\documentclass[twoside]{article}

\usepackage[preprint]{aistats2027}

\usepackage[round]{natbib}

\usepackage[utf8]{inputenc} 
\usepackage[T1]{fontenc}    
\usepackage{hyperref}   
\usepackage{url}            
\usepackage{booktabs}       
\usepackage{amsfonts}       
\usepackage{nicefrac}       
\usepackage{microtype}      
\usepackage{xcolor}         
\usepackage{graphicx}
\usepackage{subcaption}

\usepackage{todonotes}

\usepackage{amsmath,amsfonts}
\usepackage{dsfont}
\usepackage{amssymb}
\usepackage{mathtools}
\usepackage{amsthm}
\usepackage{bm}
\usepackage{bbm}
\usepackage{thmtools}
\usepackage{thm-restate}
\usepackage{algorithm,algorithmic}

\newcommand{\gapA}{\Delta^{\mathcal{A}}_{a,t}}

\newcommand{\gaphat}{\hat{\Delta}_{a,t}}
\newcommand{\q}{q_{a,t}}

\newcommand{\qhatstar}{\hat{q}_{\hat a^*,t}}

\newcommand{\z}{z_{a,t}}
\newcommand{\zstar}{z_{\hat a^*,t}}

\newcommand{\zhatstar}{\hat{z}_{\hat a^*,t}}

\DeclareMathOperator*{\argmax}{arg\,max}

\theoremstyle{plain}

\newtheorem{lemma}{Lemma}[section]
\newtheorem{corollary}{Corollary}[section]

\theoremstyle{definition}
\newtheorem{definition}{Definition}[section]
\newtheorem{assumption}{Assumption}[section]

\theoremstyle{remark}
\newtheorem{remark}{Remark}[section]
\usepackage[capitalize,noabbrev]{cleveref}
\crefname{assumption}{Assumption}{Assumptions}
\Crefname{assumption}{Assumption}{Assumptions}

\usepackage{microtype}
\title{Learning from a Bandit Learner}

\author{
}

\begin{document}

\twocolumn[

\aistatstitle{Learning from a Bandit Learner}

\aistatsauthor{ Michela Petriconi \And Nicolas Menet \And  Claire Vernade }

\aistatsaddress{ University of Tübingen \\ University of Technology Nuremberg  \And  ETH Zurich \And University of Technology Nuremberg } ]

\begin{abstract}
Classical inverse bandit methods typically aim to recover true reward parameters. In contrast, this paper proposes inferring a demonstrator’s subjective beliefs, specifically, their \emph{internal sub-optimality gap estimates}. We consider no-regret demonstrators following a randomized policy (e.g. Thompson sampling) and characterize a family of such algorithms for which \emph{per-step inversion} is possible, allowing gap estimates to be recovered directly from the arm-selection distribution. Our analysis identifies an inherent stability-coverage trade-off. Inversion becomes ill-conditioned as the demonstrator concentrates its probability mass on a single arm, while the number of informative rounds depends on the rate at which its policy converges. We delineate conditions under which inversion is reliable and establish theoretical guarantees on the resulting estimation error. Experiments on simulated data support the analysis and illustrate how these conditions affect gap recovery in practice.\looseness-1
\end{abstract}

\section{Introduction}
Imagine going to the casino with a friend. As they decide to try their luck on the slot machines, you sit back and watch, observing them progressively allocate their coins to the machines. Suppose your friend (the \emph{learner}) has recently read \emph{Bandit Algorithms} \citep{Lattimore_Szepesvári_2020} and can perform a no-regret learning strategy, but you (the \emph{observer}) do not know exactly which one. Can you infer their belief on the average rewards of the machines? 

We study the problem of \emph{learning from a learner} \citep{pmlr-v97-jacq19a} in the bandit setting, a variation of the Inverse Bandits (IB) problem~\citep{guo2022learning,guha2024one}
that we address under different assumptions and observation models. \citet{guo2022learning} study a white-box setting in which the observer knows the learner's internal decision mechanism exactly and uses it explicitly to form estimates of the \emph{true} mean rewards, that are recoverable only up to a constant due to the unidentifiability of inverse problems \citep{russell1998learning}. In contrast, we propose a formulation that relaxes the white-box assumption. Specifically, we assume that the demonstrator belongs to a broad family of randomized, no-regret bandit algorithms, and that the observer has access to the realized action sequence, with neither knowledge of the exact policy that generated it nor of the rewards associated with the played actions. However, we assume that the observer can form an estimate of the demonstrator's sampling probability. Under this structured grey-box assumption we target the learner's \emph{internal beliefs} about the means, which may not coincide with the true means \citep{halder2025stablethompsonsamplingvalid}. This shift is deliberate: these beliefs govern the behavior of a demonstrator that is still learning and, without additional information, are the only reward-related quantities that can be learned from its actions. Even if the beliefs on the rewards are biased, a second learner can leverage them to construct a better policy for the same task, in the spirit of ``better than the demonstrator'' approaches~\citep{brown2020better}.\looseness-1

We focus on randomized algorithms, which implicitly or explicitly sample arms at each round from a distribution that depends on an internal estimate of the mean rewards and the learner's uncertainty. The most prominent such algorithm is Thompson Sampling (TS) \citep{Thompson1933ONTL,russo2018tutorial}, which pulls arms according to the probability that their mean reward is larger than all others, what \citet{menet2025lite} call \emph{Probability of Maximality} (PoM). In TS, the sampling probability is implicit via an indirect posterior sampling mechanism, making it difficult to recover the belief parameter in closed form from the sampling distribution.\looseness-1

Our first contribution is a general class of randomized bandit
algorithms whose sampling probabilities admit a closed-form relationship
with their internal mean estimates. Building on the Linear-Time
Independence-based Estimator (LITE) of \citet{menet2025lite}, in~\cref{sec:LITE_sampling} we
introduce $v$-LITE Sampling ($v$-LS), which turns the LITE estimator
into a sampling rule. Under suitable assumptions, we provide regret bounds for $v$-LS in~\cref{app:gen_lite}.\looseness-1

Our second contribution is an inversion method for this class of
learners. In
\cref{sec:iLITE}, we show how the
demonstrator's sampling probabilities can be inverted to recover its
internal gap estimates. We derive inversion-error
bounds when the sampling probabilities are available only through
estimates, including under batched observations. To move beyond exactly
specified \(v\)-LS demonstrators, we define a broader family of
algorithms whose internal gaps and sampling policies are close to those
of a reference \(v\)-LS learner. For this family, the reconstruction
error is controlled jointly by the sampling-probability estimation error
and the mismatch with the reference policy. We support these theoretical
results with simulations.

\section{Preliminaries}
\label{sec:TS_v_LITE}

A stochastic bandit instance \citep{Lattimore_Szepesvári_2020} is defined as a collection of reward distributions $\nu=(\nu_a:a\in[K])$, where $[K]:=\{1,2,\dots,K\}$ denotes the set of available actions (bandit arms). For the remainder of this paper, we assume the reward noise to be zero-mean Gaussian, with variance $\alpha^2$, meaning each arm's reward distribution is $\nu_a=\mathcal N(\theta_a,\alpha^2)$, where $\theta_a:=\mathbb{E}_{R\sim\nu_a}[R]$ denotes its expected reward. In our Bayesian formulation, the vector of arm means
\(\theta=(\theta_1,\ldots,\theta_K)\) may have a correlated Gaussian
prior; that is, its prior covariance matrix need not be diagonal. Conditional on \(\theta\), the observation noise is independent
across rounds.

The learner interacts with the environment over a horizon of $T$ rounds. At each round $t$, the learner selects an action $a_t\in[K]$, then the environment samples a reward $r_t\sim \nu_{a_t}$ and reveals it to the learner.

Define $N_a(t):=\sum_{j=1}^t\mathbbm{1}\{a_j=a\}$, i.e., the number of pulls of arm $a$ at time $t$. The optimal arm is defined as $a^*:=\argmax_{a\in[K]}\theta_a$.

A learning algorithm is said to be no-regret if its cumulative regret, $\mathbb{E}[\sum_{t=1}^T \theta_{a^*} - \theta_{a_t}]$, grows sub-linearly with the horizon $T$. Classical algorithms known to achieve sub-linear regret include \emph{Upper Confidence Bound} (UCB)~\citep{auer2002finite, srinivas2009gaussian} and TS~\citep{Thompson1933ONTL,russo2018tutorial}.\looseness-1

Throughout the paper, we use the terms \emph{learner} and
\emph{demonstrator} interchangeably to refer to the online learning
algorithm that sequentially selects actions, observes their rewards, and
aims to minimize regret relative to the best arm. The
\emph{observer}, by contrast, has access only to the resulting action
trajectory \(\{a_t\}_{t=1}^T\), together with any information permitted
by the specified observation protocol, and never observes the
corresponding rewards.\looseness-1

\subsection{TS and the Probability of Maximality}
TS~\citep{Thompson1933ONTL,pmlr-v31-agrawal13a,russo2018tutorial} is a Bayesian algorithm for sequential decision-making that balances exploration and exploitation through posterior sampling. While being arguably the oldest strategy for bandit play, it has remained popular, e.g., enabling analysis of coding agents~\citep{menet2026codingagentsenvironmentinteraction} or improving robustness of LLM post-training~\citep{menet2026poetsuncertaintyawarellmoptimization}. In TS, the learner maintains a posterior distribution $p_t$ over the unknown parameters $(\theta_a)_{a\in[K]}$ governing the reward distribution of all arms. At each round $t$, the learner samples $Z_{t}\sim p_t$ from the current posterior and selects the arm 
\begin{equation}\label{eq:TS}
a_t=\argmax_{a\in[K]} Z_{a,t},\end{equation}
observes the associated reward $r_t \sim \nu_{a_t}$, and updates the posterior via Bayes rule.
In particular, under a Gaussian prior over the rewards, $\theta\sim\mathcal{N}(\mu_{0},\Sigma_{0})$ and assuming known noise variance $\alpha^2$, after selecting arm $a_t=a$ the update has a closed form $\mathcal{N}(\mu_{t+1}, \Sigma_{t+1})$. Defining $H_t \in \mathbb{R}^{K \times K}$ as a diagonal matrix with $[H_t]_{a,a} = N_a(t)$, the covariance matrix and mean vector are updated as:
\begin{align}
    &\Sigma_{t+1}=(\Sigma_0^{-1}+\alpha^{-2} H_t)^{-1},\label{eq:var}\\
    &\mu_{t+1} = \Sigma_{t+1} \left(\Sigma_0^{-1} \mu_0 + \alpha^{-2} \textstyle\sum_{s=1}^t r_s e_{a_s}\right), \label{eq:mean-update}
\end{align}
where $e_{a_s}$ is the standard basis vector for the arm pulled at round $s$. We denote the marginals for each arm as $\mathcal{N}(\mu_{a,t}, \sigma^2_{a,t})$, where $\mu_{a,t} = [\mu_t]_a$ and $\sigma_{a,t}^2 = [\Sigma_t]_{a,a}$. Importantly, $\Sigma_t$ depends solely on the initial prior covariance and the action history, making the sequence of marginal uncertainties $\sigma_{a,t}$ perfectly deterministic and observable (\cref{eq:var}). In contrast, $\mu_{t}$ depends on the observed rewards $r_s$ (\cref{eq:mean-update}), which is observable to the learner but not to the inverse learner.

The randomized mechanism of~\cref{eq:TS} was originally designed by \citet{Thompson1933ONTL} to approximate
the likelihood that one arm's mean reward is larger than those of all other arms, the so called \textit{Probability of Maximality} (PoM).
Indeed, it can be shown\footnote{Ch.~36 (Ex.~36.3) of
\citet{Lattimore_Szepesvári_2020}.} that the conditional distribution of the
\emph{identity} of the optimal arm coincides exactly with that of the action chosen by TS:
\[
\underbrace{\mathbb{P}\big(\operatorname*{arg\,max}_{b \in [K]} \theta_b = \cdot \mid \mathcal{F}_{t-1}\big)}_{:=\mathrm{PoM}(t)}
=
\underbrace{\mathbb{P}(a_t=\cdot \mid \mathcal{F}_{t-1})}_{\text{TS sampling prob.}}
\]
where $\mathcal{F}_{t-1}$ denotes all observations up to round $t-1$.

\subsection{LITE sampling}\label{sec:LITE_sampling}
Although the TS mechanism already provides samples from the PoM, direct access to it opens up new possibilities such as Entropy Search~\citep{hennig2012entropy, menet2025lite}, Bayes-optimal batch selection~\citep{fromer2025batched, menet2025lite}, policies for zero-sum two-player games~\citep{o2021variational}, and scaling up Thompson sampling via fine-tuning of LLMs~\citep{menet2025thompsonsamplingfinetuningllms}.

Therefore, we introduce a family of algorithms whose sampling probability is directly connected to the PoM.

\begin{definition}[$v$-LITE]\label{def:v_lite}
 Let $v:\mathbb R \to (0,1)$ be a strictly increasing and continuously differentiable CDF, and let $\mu_a \in \mathbb R, \sigma_a > 0$ be given location and scale parameters for each arm $a \in [K]$. We define the $v$-LITE estimator of PoM as the probability vector $q$ with\looseness-1
    \begin{equation}\label{eq:q_a}
        q_a := v\Big(\frac{\mu_a-\kappa^*}{\sigma_a}\Big) \quad \forall a \in [K],
    \end{equation}
    where the threshold $\kappa^* \in \mathbb R$ is the unique constant ensuring normalization, i.e., that $\sum_{a=1}^K q_a = 1$.
\end{definition}
Given that this approximate PoM ($q_t$) can be obtained efficiently via binary search for $\kappa^*$ \citep{menet2025lite}, a natural approximation to TS is a randomized algorithm that samples arms directly according to $v$-LITE. 

\begin{definition}[$v$-LITE Sampling, $v$-LS]\label{def:v_LS} 
    At round $t$, $v$-LS computes $\q$ according to Definition~\ref{def:v_lite}, where $(\mu_{a,t},\sigma_{a,t})$ are the posterior mean and standard deviation of the Gaussian marginals $\theta_a$. Then, $v$-LS selects an action $a_t\sim\mathrm{Cat}(q_t)$ and updates the posterior parameters.
\end{definition}

The $v$-LS family is motivated by two prominent special cases that approximate TS:
\begin{itemize}
    \item \textbf{VBOS}: Setting $v(x) = \Psi(x) := \exp(-(\sqrt{x^2 + 4} - x)^2/8)$ recovers variational Bayesian optimistic sampling (VBOS), a popular and theoretically sound no-regret approximation of TS for subgaussian rewards \citep{o2021variational, tarbouriech2023probabilistic,menet2025thompsonsamplingfinetuningllms}.
    \item \textbf{F-LITE}: Setting $v=\Phi$, where $\Phi$ is the standard normal CDF, recovers the F-LITE estimator of Gaussian PoM \citep{menet2025lite}. F-LITE is a highly accurate TS approximation derived using extreme value theory, though no-regret guarantees are not established for bandit play under $\Phi$-LS.
\end{itemize}

\(v\)-LS applies under broad distributional assumptions; its no-regret
properties are established in \cref{app:gen_lite}. Particularly, for subgaussian $\theta$ and $v(x) \!=\! \Psi(x)\!$,
it achieves $\tilde {\mathcal O}(\sqrt{T K})$ regret~\citep{tarbouriech2023probabilistic}. In the correlated Gaussian setting,
this improves to \(\widetilde{\mathcal O}(\sqrt{T\gamma_T})\)
\citep{menet2025thompsonsamplingfinetuningllms}, where \(\gamma_T\) is
the maximum information gain \citep{srinivas2009gaussian} and scales as
\(\mathcal O(d\log T)\) for an effective dimension \(d\). We write LS when referring collectively to \(v\)-LS for all admissible choices of \(v\).\looseness-1

\section{Problem Formulation}\label{sec:inverse_learning}

We consider a randomized no-regret algorithm $\mathcal A$ as a demonstrator. In each round $t$, $\mathcal A$ samples $a_t \sim p_t$ conditional on its interaction history
\(\mathcal F_{t-1}\). 

The observer has access to the complete action trajectory
\(\{a_t\}_{t=1}^T\), but not to the corresponding rewards.
We assume that the observation protocol permits the construction of estimates
\(\hat p_t\) of the conditional action distributions \(p_t\), either
from the trajectory itself or from additional action-level observations. In \cref{sec:PoM_estimated}, we quantify how
the estimation error in \(\hat p_t\) affects inversion under suitable approximation
guarantees.\looseness-1

Access to \(\hat p_t\) is behavioral rather than white-box: the
observer does not see the realized rewards, the demonstrator's internal
beliefs, or the exact mechanism mapping its history to \(p_t\). We initially
impose no relationship between the demonstrator's internal reward
estimates and its sampling probabilities. The structural assumptions
that make inversion possible are introduced in \cref{sec:iLITE}.\looseness-1

\paragraph{Objective. } 
Our goal is to recover the \emph{internal gap estimates} that drive the
demonstrator's behavior, focusing on the rounds during which it is still
actively exploring. Because the observer does not observe rewards, it
cannot identify the true optimal arm. We therefore use the empirically
most-pulled arm as an observable reference: $\hat{a}^* := \argmax_a N_a(T)$, where ties are broken arbitrarily.
For each \(a\in[K]\) and
\(t\le T\), we define the demonstrator's unobserved internal gap estimate
relative to \(\hat a^*\) as
\begin{equation}\label{eq:gapA_definition}
\gapA=\mu_{\hat a^*,t}-\mu_{a,t},
\end{equation}
where $\mu_{a,t}$ is the demonstrator's internal mean estimate for arm $a$ at time $t$. 
These stepwise gaps act as a recentering of the internal estimates around the reference arm $\hat a^*$. \looseness-1

Once the demonstrator concentrates its play on a single arm, its behavior becomes near-deterministic and therefore contains little information about the magnitude of the remaining gaps: many different internal gap values induce essentially the same near-deterministic behavior. 
In that regime, continued inversion is neither statistically stable nor necessary for explaining the demonstrator’s decisions. 
To reflect this, we restrict inversion to a set of \emph{relevant} rounds, $M^\gamma_T$, where the demonstrator is still actively exploring, defined as: 

\begin{definition}[Relevant Rounds]\label{def:safe_round}
    For a fixed threshold $\gamma \in (0,1/2)$, a round $t \le T$ belongs to the set of relevant rounds $M_T^\gamma$ if 
    $a_t \neq \hat a^*$,
    $p_{a_t,t} \ge \gamma$ and $p_{\hat a^*,t} \ge \gamma$, where $a_t$ is the pulled arm and $\hat a^*$ is the reference arm.
\end{definition}
This set depends on the demonstrator and the generated trajectory, with more details in~\cref{sec:approx_ilite,app:conditioning}.
Using the concept of relevant rounds for inversion, we define the trajectory-level inversion error as follows:\looseness-1

\begin{definition}[Trajectory-Level Inversion Error]\label{def:trajectory_error}
    The trajectory-level decision-relevant inversion error is defined as the average gap reconstruction error on relevant rounds over the entire time horizon:
    \begin{equation}\label{eq:average_objective}
        \mathcal{E}_T^\gamma := \frac{1}{T}\sum_{t=1}^T \mathds{1}\{t\in M^\gamma_T\} \, \big| \hat \Delta_{a_t,t} - \Delta_{a_t,t}^\mathcal{A} \big|,
    \end{equation}
    where $\hat \Delta_{a_t,t}$ is the reconstructed gap of arm $a_t$ at time $t$.\looseness-1
\end{definition}

This objective can be interpreted as follows. 
Define the realized relevant coverage as
\begin{equation}
    C_{T}^\gamma:=\frac{|M_T^\gamma|}{T}, 
\end{equation}
and the average reconstruction error over relevant rounds as
\begin{equation}
    L_T^\gamma:=\frac{1}{\max\{|M_T^\gamma|, 1\}} \sum_{t=1}^T \mathds{1}\{t\in M^\gamma_T\} | \hat \Delta_{a_t,t} - \Delta_{a_t,t}^\mathcal{A} |.
\end{equation}
Hence, 
\[\mathcal{E}_T^\gamma=C_{T}^\gamma L_T^\gamma,\]
is a coverage-weighted error. 
Importantly, $C_T^\gamma$ is determined by the demonstrator and the realized trajectory, not by the observer. Therefore, for a fixed demonstrator, horizon, and threshold, the observer cannot improve $\mathcal{E}_T^\gamma$ by changing which rounds are inverted.\looseness-1

Coverage may decrease as the demonstrator concentrates on
\(\hat a^*\), and can vanish in the limiting case. This is
consistent with our objective: when \(a_t=\hat a^*\), both the
true and estimated relative to
\(\hat a^*\) are identically zero, while gaps involving
vanishing-probability actions cannot be stably recovered from behavior.\looseness-1

\paragraph{Related work}
Our demonstrator-focused metric \eqref{eq:average_objective} deliberately departs from standard Inverse Bandits (IB)~\citep{guo2022learning}, which seek to recover the true underlying means $\hat{\theta}_a$. As shown by \citet{guo2022learning}, reconstructing $\hat{\theta}_a$ suffers an information-theoretic lower bound of: 
\[\mathbb{E}[|\hat{\theta}_a-\theta_a|]\ge \frac{1}{16}\min\Bigg \{ \frac{1}{\sqrt{\mathbb{E}[N_{a}(T)]}}, 1 \Bigg \}.\]
For standard no-regret demonstrators, where a suboptimal arm is pulled only \(O(\log T)\) times, the lower bound
decreases only as \(\Omega(1/\sqrt{\log T})\).
\citet{guo2022learning} provide nearly matching estimators for
parameterized UCB~\citep{auer2002finite,srinivas2009gaussian} and successive-arm-elimination~\citep{JMLR:v7:evendar06a} demonstrators. Although
the logarithmic-regret case remains consistent, achieving a polynomially
decreasing estimation error requires deliberately inflated exploration,
resulting in \(O(T^\alpha)\) regret for some \(\alpha>0\).
Our objective is different: rather than recovering the stationary environment means, we reconstruct the demonstrator's evolving internal
gap estimates that govern its actions. Consequently, our method does not require modifying the demonstrator to induce additional exploration for
true-reward identification. It applies to standard randomized no-regret demonstrators using action-level behavioral observations and permits approximate inversion under policy misspecification (Section~\ref{sec:approx_ilite}).\looseness-1

\section{\texttt{\lowercase{i}LITE}: Inverting LS}\label{sec:iLITE}
We start by illustrating the idea of our method in an idealized setting to build intuition. 
\subsection{Inversion in an idealized setting}\label{sec:idealized}
We define \(v\)-LS (\cref{def:v_LS}) as our \emph{reference demonstrator}. Under the additional assumptions below, we show that it can be perfectly inverted; we relax these assumptions in subsequent sections.\looseness-1
    \begin{assumption}\label{ass:dem_v}
        The demonstrator is \emph{exactly} $v$-LS.\looseness-1 
    \end{assumption}
    \begin{assumption}\label{ass:full_knowledge}
        At each time $t$, the observer has full knowledge of the sampling distribution $q_t$ (\cref{eq:q_a}) of the demonstrator, i.e., $a_t \sim Cat(q_t)$.
    \end{assumption}
    \begin{assumption}\label{ass:parameters}
       The observer knows the initial prior covariance $\Sigma_0$, the reward noise variance $\alpha^2$.
    \end{assumption}

    The sampling probability $q_t$ reveals a directly computable scaled and shifted proxy for the mean posterior:\looseness-1
    \begin{align}\label{eq:zat}
        \z:=v^{-1}(\q)=\frac{\mu_{a,t}-\kappa^*_t}{\sigma_{a,t}}.
\end{align}
As a result, it holds that
\begin{equation}
    \mu_{a,t}=z_{a,t}\sigma_{a,t}+\kappa^*_t,
\end{equation}
where $\sigma_{a,t}^2$ is the reward variance of arm $a$ at time $t$ and can be computed, as it only depends on the pulled arms up to time $t$ and the prior, see~\cref{eq:var}. Even if 
$\kappa_t^*$ is unknown, it is only used internally by the demonstrator to normalize $q_t$. As a result, the internal gaps $\gapA$ (\cref{eq:gapA_definition}) w.r.t. the most-pulled arm $\hat a^*:=\argmax_a N_a(T)$ can be recovered exactly:
\begin{equation}\label{eq:exact_inv}
    \begin{split}
        \gaphat=\gapA&:=\zstar\sigma_{\hat a^*,t}-\z\sigma_{a,t}.
    \end{split}
\end{equation} 
In this idealized setting, $\mathcal{E}^{\gamma}_T=0$ $\forall\>\gamma\in(0,1)$, i.e., perfect inversion can be performed.

\subsection{Approximate iLITE}\label{sec:approx_ilite}
This perfect inversion relies on two major assumptions: the learner is \emph{exactly} the reference demonstrator (\cref{ass:dem_v}), and its sampling probabilities are known \emph{without approximations} (\cref{ass:full_knowledge}). In the following subsections, we relax these two assumptions. The proofs of this section are deferred to~\cref{app:proofs_ilite}.

\subsubsection{PoM estimated}\label{sec:PoM_estimated}
We first retain \cref{ass:dem_v} but relax
\cref{ass:full_knowledge}. The observer now has access only to estimates
\(\widehat q_t\) of the \(v\)-LS distributions \(q_t\).
For $\eta>0$, we define the uniform approximation event
\begin{equation}
    \mathcal G_\eta
    :=
    \left\{
        \max_{t\in[T]}\max_{a\in[K]}
        \left|
            \hat q_{a,t}-q_{a,t}
        \right|
        \leq\eta
    \right\}.
    \label{eq:uniform-approximation-event}
\end{equation}
Conditioned on this event occuring, we bound the inversion error
\(\mathcal E_T^\gamma\) from \cref{eq:average_objective} in terms of
\(\eta\). The result is independent of how \(\hat q_t\) is
constructed: any observation protocol that ensures
\(\mathbb P(\mathcal G_\eta)\geq 1-\delta\) yields the corresponding
high-probability inversion guarantee.\looseness-1

On \(\mathcal G_\eta\), inversion proceeds as in the idealized setting:
\begin{equation}
    \hat\Delta_{a,t}
    =
    \hat z_{\hat a^*,t}
    \sigma_{\hat a^*,t}
    -
    \hat z_{a,t}\sigma_{a,t},
    \qquad
    \hat z_{a,t}
    :=
    v^{-1}(\hat q_{a,t}),
    \label{eq:approximate-gap-inversion}
\end{equation}
where the variances $\sigma_{a,t}^2$ remain
computable from the observed action trajectory. We refer to this
inversion procedure as \texttt{iLITE} and summarize it in
\cref{algo:gap_estimation_2}.\looseness-1

The oracle relevance set $M_T^\gamma$ from
\cref{def:safe_round} depends on the unknown probabilities \(q_t\) and
is therefore unavailable to the observer. We instead define, for $\rho>0$, the corresponding empirical relevance
set
\begin{equation}
    \widehat M_T^\rho
    :=
    \left\{
        t\in[T]:
        a_t\neq\widehat a^*,\ 
        \widehat q_{a_t,t}\geq\rho,\ 
        \widehat q_{\widehat a^*,t}\geq\rho
    \right\}.
    \label{eq:empirical-relevance-set}
\end{equation}
On $\mathcal G_\eta$, the true and empirical relevance sets satisfy
\begin{equation}
    \widehat M_T^{\gamma+\eta}
    \subseteq
    M_T^\gamma
    \subseteq
    \widehat M_T^{\gamma-\eta}.
    \label{eq:relevance-set-sandwich}
\end{equation}

Thus, by computing gap estimates on
\(\widehat M_T^{\gamma-\eta}\), \texttt{iLITE} covers every oracle-relevant
round. This allows us to bound the oracle error
\(\mathcal E_T^\gamma\) without requiring the observer to know
\(M_T^\gamma\).

Restricting inversion to relevant rounds is necessary because
\(v^{-1}\) is not globally Lipschitz on \((0,1)\): for the reference
CDFs, \(\lvert(v^{-1})'(u)\rvert\) diverges as \(u\)
approaches \(0\) or \(1\), potentially amplifying small probability
errors. We refer to \cref{app:conditioning} for a detailed illustration of this non-Lipschitz behavior.

Therefore, we restrict gap recovery to a domain where $v^{-1}$ is guaranteed to be locally Lipschitz with a finite constant
\begin{equation}
    c_{\beta}
    :=
    \sup_{u\in[\beta,\,1-\beta]}
    \left|(v^{-1})'(u)\right|\text{, where } \beta\in(0,1/2)
    \label{eq:local-lipschitz}
\end{equation}

In particular, if \(t\in \widehat M_T^{\gamma-\eta}\), then, for each
\(j\in\{a_t,\widehat a^*\}\), it holds that $\hat q_{j,t}\in[\gamma-\eta,1-\gamma+\eta].$
On \(\mathcal G_\eta\), this implies
\[
    q_{j,t},\hat q_{j,t}
    \in[\gamma-2\eta,1-\gamma+2\eta] \> \text{and } \beta=\gamma-2\eta.
\]

In \cref{algo:gap_estimation_2}, we set
\(\rho=\gamma-\eta\), so that \texttt{iLITE} computes gap estimates on
\(\widehat M_T^{\rho}\), which contains \(M_T^\gamma\) on
\(\mathcal G_\eta\). \looseness-1

\begin{algorithm}[H]
\caption{\texttt{iLITE}}
\label{algo:gap_estimation_2}
\begin{algorithmic}[1]
\REQUIRE Trajectory $\{a_{t}\}_{t\le T}$, CDF $v$, $\Sigma_{0},\alpha^2$, $\rho \le 1/2$\looseness-1
\STATE $\hat a^* \gets \argmax_{a \in [K]} N_a(T)$  {\small\textit{// Empirical best arm}}
\FOR{$t\in \{1,2,\dots,T\}$}
    \STATE \small\textit{// Obtain PoM estimates}
    \STATE Compute $\hat q_{a_t,t},\qhatstar$
    \IF{
$a_t\neq\widehat a^*$ and
$ \hat q_{a_t,t},
\hat q_{\widehat a^*,t}
\geq\rho$
}
        \STATE \small\textit{//Compute posterior variances}
        \STATE Compute $\sigma_{a_t,t},\sigma_{\hat a^*,t}$ as in Equation~\eqref{eq:var} 
        \STATE $\hat z_{a_t,t}\gets v^{-1}(\hat q_{a_t,t})$ and  $\zhatstar\gets v^{-1}(\qhatstar)$
        \STATE $\hat\Delta_{a_t,t} \gets \zhatstar\sigma_{\hat a^*,t}-\hat z_{a_t,t}\sigma_{a_t,t}$
    \ENDIF
\ENDFOR
\STATE \textbf{Return}: $\hat\Delta_{a_t,t}$ for all relevant rounds
\end{algorithmic}
\end{algorithm}

The following lemma quantifies how an error in the estimated sampling
probabilities propagates through the inverse map $v^{-1}$.

\begin{restatable}{lemma}{gapeps}\label{lem:gaps_eps} 
Let 
$0<\eta<\rho\le1/2$. Suppose that
$t\in \widehat M_T^{\rho}$ is a relevant inversion round. On the event $\mathcal G_\eta$,\looseness-1
   \begin{equation}
       |\hat\Delta_{a_t,t}-\Delta^{LS}_{a_t,t}|\le (|\sigma_{\hat a^*,t}|+|\sigma_{a_t,t}|)c_{\rho-\eta} \eta,
    \end{equation}
where $c_{\rho-\eta}$ is the Lipschitz constant of $v^{-1}$ on the truncated domain $[\rho-\eta, 1-\rho+\eta]$, defined in \cref{eq:local-lipschitz}.\looseness-1
\end{restatable}

Aggregating the local error over the relevant rounds gives the following
trajectory-level guarantee:

\begin{restatable}{theorem}{approxlite}
\label{thm:approx_lite}
    Assume that $\mathcal{A}$ satisfies~\cref{ass:dem_v}, that
$\sigma_{a,0}^2\leq 1$ for every $a\in[K]$, and
$0<2\eta<\gamma\le1/2$. Then, on
$\mathcal G_\eta$, every realized action trajectory satisfies
    \begin{align}\label{eq:approx-lite-bound}
    \mathcal{E}_T^\gamma&
    \le 
    c_{\gamma-2\eta}\cdot \eta \Big(\frac{(\alpha\sqrt{2}+1)\sqrt{\gamma_T}+2\alpha}{\sqrt{T}}+\frac{t^*+1}{T}\Big). 
    \end{align}
    Here, $c_{\gamma-2\eta}$ is the Lipschitz constant of $v^{-1}$ on the truncated domain $[\gamma-2\eta, 1-\gamma+2\eta]$, defined in \cref{eq:local-lipschitz}, $\gamma_T$ is the maximal information gain\footnote{This is an information-theoretic quantity defined in~\cref{lem:4}, that only depends on the posterior variances along the trajectory and problem-dependent quantities.}, and $t^*:=\sum_{a\neq\hat a^*}N_a(T)$ counts the number of exploratory rounds.\looseness-1
\end{restatable}
Consequently, if
$\mathbb P(\mathcal G_\eta)\geq 1-\delta$, then
\cref{eq:approx-lite-bound} holds with probability at least
$1-\delta$. If $\mathcal G_\eta$ holds almost surely, taking
expectations gives the following result.

\begin{corollary}
    Under the assumptions of \cref{thm:approx_lite}, if $\mathcal G_\eta$ holds almost surely, then
    \begin{align*}
    \mathbb{E}[\mathcal{E}_T^\gamma]&
    \le 
    c_{\gamma-2\eta}\cdot \eta \Big(\frac{(\alpha\sqrt{2}+1)\sqrt{\gamma_T}+2\alpha}{\sqrt{T}}+\frac{\mathbb{E}[t^*]+1}{T}\Big). 
    \end{align*}
\end{corollary}

The bound separates two sources of error. The factor
\(c_{\gamma-2\eta}\eta\) captures the amplification of the
sampling-probability error by \(v^{-1}\). The remaining terms reflect
the demonstrator's learning dynamics: the information-gain contribution
vanishes as \(\sqrt{\gamma_T/T}\), while \(t^*/T\) is the fraction
of rounds on which the demonstrator selects an arm other than the
hindsight reference arm. In particular,
\(\gamma_T\in\mathcal O(d\log T)\) scales with the effective dimension of the correlated arm space for $d$ degrees of freedom.\looseness-1

To relate \(\mathbb{E}[t^*]/T\) to regret, suppose additionally that the
bandit instance has a unique optimal arm $a^*$ and define $\Delta_{\min}
    :=
    \min_{a\neq a^*}(\theta_{a^*}-\theta_a)>0$.
Let
\[
    S_T:=\sum_{a\neq a^*}N_a(T)
\]
be the total number of suboptimal pulls. Then $t^*\leq S_T$. Indeed, if
$\hat a^*=a^*$, the two quantities are equal. Otherwise,
$N_{\hat a^*}(T)\geq N_{a^*}(T)$ because $\hat a^*$ is the
most-pulled arm, and hence
\[
    t^*
    =
    N_{a^*}(T)+S_T-N_{\hat a^*}(T)
    \leq S_T.
\]
Therefore, if the demonstrator has expected regret bounded by
$R(T)$,
\begin{equation}
    \mathbb E[t^*]
    \leq
    \mathbb E[S_T]
    \leq
    \frac{R(T)}{\Delta_{\min}}.
\end{equation}
Thus, for any $v$ such that $v$-LS is a no-regret demonstrator satisfying $R(T) \in o(T)$, the
remainder $\mathbb E[t^*]/T$ vanishes asymptotically.

\paragraph{Batched Bandit Inversion} 
The preceding result is agnostic to how \(\hat q_t\) is constructed
and requires only the uniform approximation guarantee. One possibility
is to fit a parameterized temporal model \(q_t(\phi)\) to the observed
actions using maximum likelihood estimation (MLE), or maximum a
posteriori (MAP) estimation when a prior on \(\phi\) is available. These
methods share information across rounds through the assumed temporal
model, since a single action does not by itself identify \(q_t\).\looseness-1

A simpler and widely used in practice alternative arises under batched play, where the sampling
distribution remains fixed within each batch and can therefore be
estimated directly from empirical action frequencies
\citep{kalkanli2021batched}. We focus on this setting.\looseness-1

We introduce a batched version of \(v\)-LS in which the learner updates
its posterior, and hence its sampling distribution, only between
batches. The same estimation method applies to any randomized batched
demonstrator whose sampling distribution remains fixed within each
batch.\looseness-1

Suppose that $n$ divides $T$, and let $B=T/n$ denote the number of batches of size $n$. At the beginning of batch $b\in[B]$, the batched $v$-LS computes a sampling distribution $q_b$ from its current posterior.
Conditionally on the history available before the batch, it then samples
\[
    a_{b,1},\ldots,a_{b,n}
    \mid \mathcal F_{(b-1)n}
    \overset{\mathrm{i.i.d.}}{\sim}
    \operatorname{Cat}(q_b),
\]
and updates its posterior using the $n$ collected rewards only at the end of the batch.
Consequently, $q_b$ remains fixed throughout batch $b$, and the observer obtains $n$ conditionally independent samples from
$\operatorname{Cat}(q_b)$. The sampling probabilities can thus be estimated using the empirical action frequencies
\begin{equation}\label{eq:q_b}
    \hat q_{a,b}
    :=
    \frac{1}{n}
    \sum_{j=1}^{n}\mathds{1}\{a_{b,j}=a\},
    \qquad a\in[K].
\end{equation}
After the batch has been observed, the same estimate $\hat q_b$
can be used as $\hat q_t$ for every round $t$ belonging to batch
$b$.\looseness-1

Using Hoeffding's inequality together with a union bound, the next result provides an approximation guarantee that holds uniformly over all arms and batches

\begin{restatable}{lemma}{estimatedPoM}
\label{lem:estimate_p}
For any $\delta\in(0,1)$, define
\begin{equation}
    \eta_{n,\delta}
    :=
    \sqrt{
        \frac{\log(2KB/\delta)}{2n}
    }.
    \label{eq:batched_eta}
\end{equation}
Then
\begin{equation}
    \mathbb{P}\left(
        \max_{b\in[B]}
        \max_{a\in[K]}
        \left|
            \hat q_{a,b}-q_{a,b}
        \right|
        \leq \eta_{n,\delta}
    \right)
    \geq 1-\delta.
    \label{eq:uniform_batched_event}
\end{equation}
\end{restatable}

Define the uniform approximation event
\begin{equation}
    \mathcal G_{n,\delta}
    :=
    \left\{
        \max_{b\in[B]}
        \max_{a\in[K]}
        \left|
            \hat q_{a,b}-q_{a,b}
        \right|
        \leq \eta_{n,\delta}
    \right\}.
    \label{eq:batched_good_event}
\end{equation}

Combining \cref{lem:estimate_p} with the local stability of the inverse
map gives a trajectory-level guarantee.

\begin{restatable}{theorem}{batchederror}
    Let $\mathcal{A}$ be exactly batched $v$-LS,
$\sigma_{a,0}^2\leq1$ for every $a\in[K]$, and the Gaussian
observation-noise variance be $\alpha^2$. Let
$0<2\eta_{n,\delta}<\gamma\le1/2$.
    On the event $\mathcal G_{n,\delta}$, every realized action trajectory
satisfies
\begin{equation}
\label{eq:batched_inversion_bound}
\begin{aligned}
\mathcal E_T^\gamma
&\le
c_{\gamma-2\eta_{n,\delta}}\cdot\eta_{n,\delta}
\Biggl[\frac{
(\alpha\sqrt 2+\sqrt n)\sqrt{\gamma_B}+2\alpha}{\sqrt{T}}
\\[-0.2em]
&\hspace{7em}
+\frac{t^*+n(1+B_\alpha)}{T}
\Biggr].
\end{aligned}
\end{equation}
    Here, $B_\alpha:=\min\{1,\alpha\}$, $\gamma_B$ is the maximal information gain, $c_{\gamma-2\eta_{n,\delta}}$ is defined as in~\cref{eq:local-lipschitz}, and $t^*:=\sum_{a\neq\hat a^*}N_a(T)$. 
Consequently, by \cref{lem:estimate_p}, the bound in
\cref{eq:batched_inversion_bound} holds with probability at least
$1-\delta$.
\end{restatable}

The bound reveals a trade-off in the batch size \(n\). Larger batches
provide more samples from each fixed sampling distribution, reducing the
estimation error as
\(\eta_{n,\delta}\in O(n^{-1/2})\), up to logarithmic factors. However,
delaying posterior updates introduces the explicit factors \(\sqrt n\)
and \(n/T\) in \cref{eq:batched_inversion_bound}. Thus, larger batches
improve probability estimation at the cost of less frequent learning
updates.\looseness-1

\subsubsection{$(\varepsilon,\tilde\varepsilon)$-LS-family}
We relax \cref{ass:dem_v} to allow demonstrators that differ from the
reference $v$-LS learner. For every \(t\in M_T^\gamma\), the
inversion error decomposes as 
\begin{equation}
  |\hat\Delta_{a_t,t}-\Delta^{\mathcal{A}}_{a_t,t}|\le|\hat\Delta_{a_t,t}-\Delta^{LS}_{a_t,t}|+|\Delta^{LS}_{a_t,t}-\Delta^{\mathcal{A}}_{a_t,t}|.  
\end{equation}
The two terms capture, respectively, inversion error relative to the
reference and discrepancy between the
internal gaps of the reference and actual demonstrators.
To quantify these discrepancies we introduce the following definitions.
\begin{definition}[Gap mismatch]\label{ass:same_gaps}
Let \(\mathcal A\) and the reference LS learner be evaluated on the same
realized history \(\mathcal F_{t-1}\). Define
\[
\varepsilon_{\mathrm{gap}}(\mathcal A;T)
:=
\max_{t\in[T]}
\left|
\Delta^{\mathrm{LS}}_{a_t,t}
-\Delta^{\mathcal A}_{a_t,t}
\right|.
\]
\end{definition}

\begin{definition}[Sampling-policy mismatch]
Let $p_t$ denote the sampling distribution of $\mathcal{A}$ and let
$q_t$ denote the sampling distribution of the LS reference, with both
policies evaluated conditionally on the same realized history
$\mathcal{F}_{t-1}$. Define
\[
    \varepsilon_{\mathrm{pol}}(\mathcal{A};T)
    :=
    \max_{t \leq T}\max_{a\in[K]}
    \left|p_{a,t}-q_{a,t}\right|.
\]
Since $p_{a,t},q_{a,t}\in[0,1]$, we always have $
    \varepsilon_{\mathrm{pol}}(\mathcal{A};T)\in[0,1]$.
Thus, this quantity is well-defined for every $\mathcal{A}$.\looseness-1
\end{definition}

\begin{definition}[\((\varepsilon,\tilde\varepsilon)\)-LS family]\label{def:demos_family}
An algorithm \(\mathcal A\) belongs to the
\((\varepsilon,\tilde\varepsilon)\)-LS family if, almost surely,
\[
\varepsilon_{\mathrm{pol}}(\mathcal A;T)\leq\varepsilon
\qquad\text{and}\qquad
\varepsilon_{\mathrm{gap}}(\mathcal A;T)\leq\tilde\varepsilon.
\]
\end{definition}

These families are nested in both tolerances. The choices
\(\varepsilon=0\) and \(\tilde\varepsilon=0\) correspond to exact
agreement with the reference sampling policy and internal gaps,
respectively, while \(\varepsilon=1\) imposes no policy restriction.
The resulting inversion guarantees are informative when both
tolerances are small, as shown below.

Conditional on the same realized history, the gap mismatch $\tilde \varepsilon$ is zero for
any algorithm using the same Gaussian prior and posterior-update rule
as the reference learner. This includes TS and other
posterior-based acquisition rules such as EI~\citep{EI}, ES~\citep{malakar2011entropy}, and MES~\citep{pmlr-v70-wang17e}. Consequently, TS belongs to an \((\varepsilon,0)\)-LS family, where the smallest
admissible \(\varepsilon\) measures the discrepancy between its sampling
probabilities and their \(v\)-LS approximation.

\begin{restatable}{theorem}{epsbound}
    Assume that $\mathcal{A}$ belongs to the $(\varepsilon,\tilde\varepsilon)$-LS-family, and $0<2\varepsilon<\gamma\le1/2$. Then, under the same variance and information-gain conditions as for the exact \(v\)-LS result, almost surely, it holds that
    \begin{equation}
        \mathcal{E}^\gamma_T\le c_{\gamma-2\varepsilon}\cdot \varepsilon \Big(\frac{(\alpha\sqrt{2}+1)\sqrt{\gamma_T}+2\alpha}{\sqrt{T}}+\frac{t^*+1}{T}\Big)+\tilde\varepsilon.
    \end{equation}
    Here, $c_{\gamma-2\varepsilon}$ is the constant defined in
\cref{eq:local-lipschitz}, evaluated at $\gamma-2\varepsilon$, $\gamma_T$ is the maximal information gain and $t^*:=\sum_{a\neq\hat a^*}N_a(T)$.
\end{restatable}

When only estimates \(\{\hat p_t\}_{t=1}^T\) of the demonstrator's
sampling distributions are available, this result can be combined with
the probability-estimation guarantees of the previous subsection.\looseness-1

\section{Experiments}
\label{sec:experiments}

We empirically investigate the two sources of inversion error predicted by our analysis: the
sampling probability estimation and the reference policy misspecification.
Additional experiments are deferred to
\cref{app:experiments}, where we also analyze misspecification of the prior $\Sigma_0$, i.e. we relax~\cref{ass:parameters}.

We use \(K=4\) Gaussian arms with means
\((0,1/3,2/3,1)\), observation variance \(\alpha^2=1\), independent
\(\mathcal N(0,1)\) priors, and \(100\) simulations.
\looseness-1

\noindent\textbf{Batched probability estimation.}
We use batched \(\Psi\)-LS to illustrate the conditioning--coverage
trade-off under an easily implementable observation protocol. The learner's
posterior and sampling distribution remain fixed within each batch, and
we estimate \(q\) using empirical action frequencies. Exact-\(q\)
inversion would recover the internal gaps up to numerical precision. The horizon considered here is $T=20000$.

\begin{figure}[t]
    \centering
    \includegraphics[width=\linewidth]
    {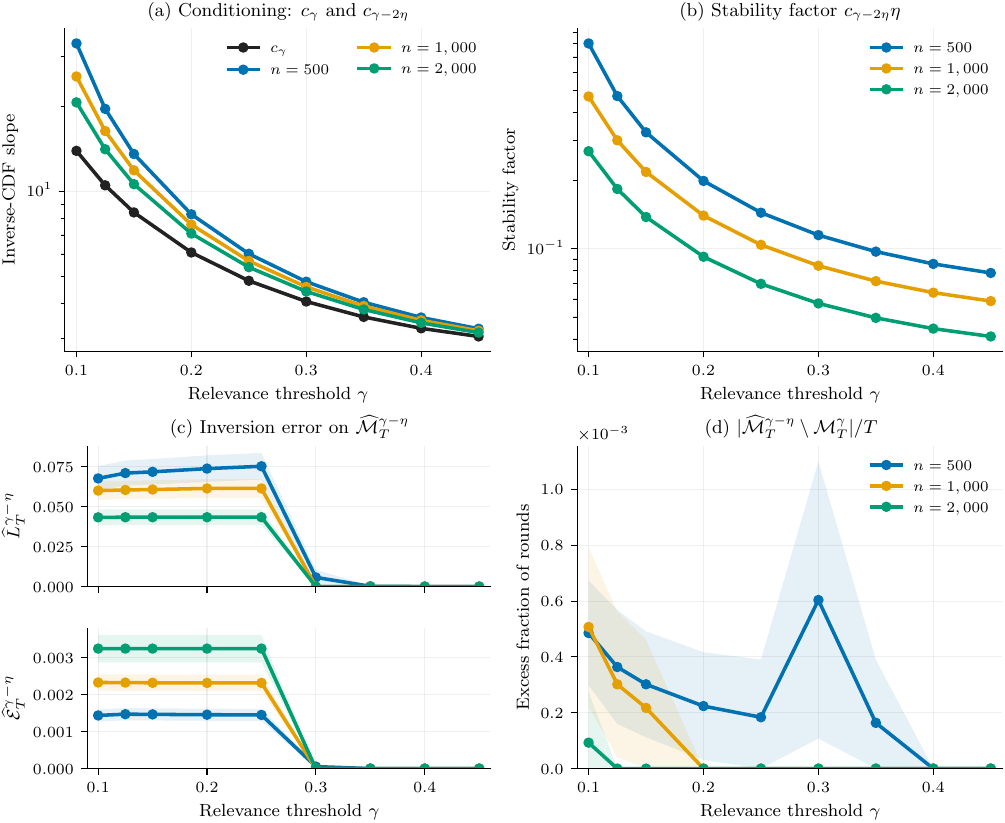}
    \caption{Batched $\Psi$-LS for
    $n\in\{500,1000,2000\}$, over $100$ trajectories. Here, $\eta$ is the
    realized uniform error of the within-batch frequency estimator.
    Panel~(a) compares $c_\gamma$ with $c_{\gamma-2\eta}$, and
    Panel~(b) reports the stability factor $c_{\gamma-2\eta}\eta$.
    Panel~(c) shows the conditional and coverage-weighted inversion errors,
    $\widehat L_T^{\gamma-\eta}$ and
    $\widehat{\mathcal E}_T^{\gamma-\eta}$, computed on the practical
    relevance set $\widehat{\mathcal M}_T^{\gamma-\eta}$.
    Panel~(d) reports the excess coverage
    $|\widehat{\mathcal M}_T^{\gamma-\eta}
    \setminus\mathcal M_T^\gamma|/T$.
    Quantities are reported only for trajectories satisfying
    $2\eta<\gamma$; shaded regions in Panels~(c)--(d) are pointwise
    $95\%$ confidence intervals.}
    \label{fig:batched-psi-ls}
\end{figure}

\Cref{fig:batched-psi-ls} illustrates the two effects predicted by our
analysis. Increasing \(\gamma\) improves conditioning by moving the
inverted probabilities away from the boundary, but reduces coverage.
Increasing the batch size \(n\) improves the frequency estimate, bringing
\(c_{\gamma-2\eta}\) closer to \(c_\gamma\) and reducing
\(c_{\gamma-2\eta}\eta\). The excess
coverage between exact and empirical relevance sets measures the resulting relevance-set error. Because \(n\)
also determines how often the learner updates, differences across batch
sizes additionally reflect the effect of delayed feedback.

\noindent\textbf{End-to-end inversion.}
Finally, we invert TS, Maillard sampling (MS)
\citep{maillard:tel-00845410,bian2022maillard}, and
Boltzmann-Gumbel exploration (BGE) \citep{cesa2017boltzmann} using their exact
sampling probabilities at $\gamma=0.4$ and horizon \(T=2000\) . For each trajectory, we replay the
$\Psi$-LS reference on the same history and retain trajectories satisfying
$2\varepsilon<\gamma$. We also invert a $\Phi$-LS demonstrator through
$\Psi^{-1}$. Because the two LS demonstrators have identical posterior gaps
on the same history, the resulting error isolates sampling-policy and
inverse-link misspecification.

\begin{figure}[t]
    \centering
    \includegraphics[width=\linewidth]
    {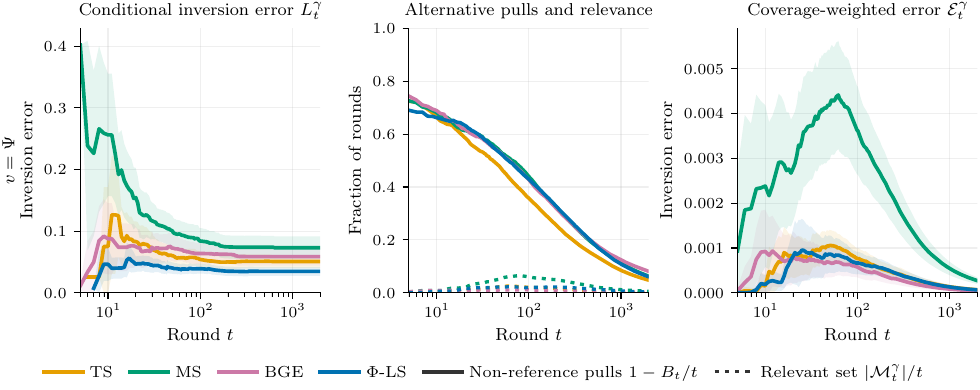}
    \caption{End-to-end inversion for $v=\Psi$, $K=4$, $T=2000$, and
    $\gamma=0.4$. The panels show conditional inversion error
    $L_t^\gamma$, the fractions of non-reference and relevant rounds, and
    coverage-weighted error $\mathcal E_t^\gamma$. Here
    $B_t=\sum_{s=1}^t\mathbbm{1}\{A_s=\widehat a^*\}$ counts pulls of
    the estimated best, or reference, arm. Solid and dotted curves in the
    middle panel denote $1-B_t/t$ and $|\mathcal M_t^\gamma|/t$,
    respectively. Shading gives pointwise $95\%$ confidence intervals over
    theorem-valid trajectories.}
    \label{fig:end-to-end-inversion-psi}
\end{figure}

The middle panel displays the start-learn-converge behavior described by
\citet{garivier2019explore}. Initially, the algorithms distribute pulls
across the arms; during learning, alternative pulls remain frequent and the
relevance set grows; finally, the algorithms concentrate on
$\widehat a^*$. In this last phase, both $1-B_t/t$ and
$|\mathcal M_t^\gamma|/t$ approach zero. Thus, vanishing coverage is mainly
explained by the algorithm repeatedly selecting its estimated best arm,
rather than only by alternative pulls being removed by the relevance
threshold. On these best-arm rounds the target gap is zero by definition.
Consequently, the conditional error stabilizes once few new relevant rounds
are added, while $\mathcal E_t^\gamma$ decreases. The $\Phi$-LS curve also
shows that using the wrong inverse link produces nonzero but bounded error
in this experiment, providing empirical evidence of robustness to moderate
misspecification of $v$. To complement these robustness results, in~\cref{sec:epsilon-ls-experiment} we report empirical estimates of $\varepsilon$ and $\tilde{\varepsilon}$ for each algorithm.\looseness-1

\section{Conclusion}

We introduce an inverse-bandit framework for recovering a demonstrator's
internal suboptimality gaps rather than the environment's true rewards.
We define the invertible \(v\)-LS class, extend it to the
\((\varepsilon,\tilde\varepsilon)\)-LS family, and provide trajectory-level
guarantees under approximate sampling probability estimation, policy misspecification, and
batched bandit play.\looseness-1

Our results expose an intrinsic exploration-inversion trade-off. While
the demonstrator explores, its sampling probabilities reveal its
internal gaps; as its policy concentrates, different gap magnitudes
produce nearly identical behavior, making stable inversion impossible.
This limitation is inherent to the information contained in the action
trajectory and reflected in the coverage $C_T^\gamma$, describing the fraction of invertible rounds.

\noindent\textbf{Limitations and future work.}
Future work includes exploiting temporal continuity between rounds, which
could enable probability estimation under weaker observation protocols. We empirically show one possibility in~\cref{app:sliding_windows}.
Other directions include extending the framework to other
structured policies, such as KL-UCB variants
\citep{garivier2011kl}, and studying its privacy implications. In
particular, recovered gaps may reveal sensitive preferences of the users, connecting
algorithm inversion to differential privacy
\citep{ou2024thompson} remains an open question\looseness-1.

\section*{Acknowledgments}
C. Vernade is funded by the Deutsche Forschungsgemeinschaft (DFG) under both the project 468806714 of the Emmy Noether Programme and under Germany’s Excellence Strategy – EXC number 2064/1 – Project number 390727645. CV also gratefully acknowledges funding from the European Union (ERC, ConSequentIAL, 101165883). Views and opinions expressed are however those of the author(s) only and do not necessarily reflect those of the European Union or the European Research Council. Neither the European Union nor the granting authority can be held responsible for them). CV and M. Petriconi also thank the international Max Planck Research School for Intelligent Systems (IMPRS-IS).

\bibliographystyle{plainnat}
\bibliography{inverse_bandits}


\clearpage
\appendix
\thispagestyle{empty}
\onecolumn
\section{Appendix}
\subsection{Stability of inverting $v$}\label[appendix]{app:conditioning}\label[appendix]{app:lite_sampling}

Because $v$ is a valid CDF, it must satisfy $v^{-1}(q)\rightarrow -\infty$ as $q\downarrow0$ and $v^{-1}(q)\rightarrow\infty$ as $q\uparrow 1$. Consequently, $v^{-1}$ cannot be globally Lipschitz on the open interval $(0,1)$. More precisely, suppose that $v$ is continuously differentiable, with $v'(x)>0$, and that $v'(x)\rightarrow0$ for $x\rightarrow\pm\infty$. By the inverse function theorem, the derivative of the inverse is given by
\[(v^{-1})'(q)=\frac{1}{v'(v^{-1}(q))},\]
and therefore $|(v^{-1})'(q)|\rightarrow\infty$ as $q\downarrow0$ or $q\uparrow1$. 

In the context of inverse bandits, the external observer only has access to an empirical estimate $\hat q$ of the demonstrator's true sampling probability $q$. While $\hat q$ typically concentrates around $q$ as the observation horizon grows, the non-Lipschitz nature of the inverse CDF means that even small estimation errors can be severely amplified during inversion. Near the boundary of the probability simplex, an arbitrarily small perturbation of $q$ may produce an arbitrarily large perturbation of $v^{-1}(q)$, as illustrated in~\cref{fig:error_amplification}.
\begin{figure}[h]
    \centering
    \includegraphics[width=0.6\linewidth]{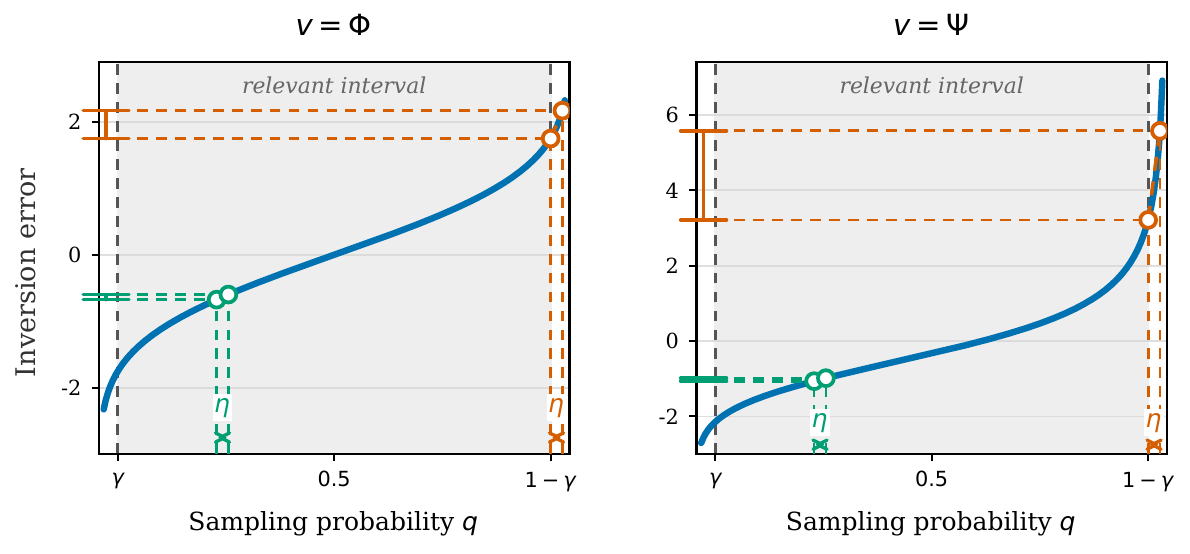}
    \caption{Error amplification under inversion for two reference CDFs, $v=\Phi$ (Gaussian) and $v=\Psi$. The same estimation error $|\hat q-q|<\eta$ is harmless inside the relevant interval $(\gamma,1-\gamma)$ but produces unbounded gap errors outside it.}
    \label{fig:error_amplification}
\end{figure}
The inverse mapping is nevertheless well-conditioned on every compact subinterval of $(0,1)$. For any margin parameter $\gamma\in(0,1)$, define the worst-case slope of $v^{-1}$ within an $\eta$-neighborhood of the relevant interval:
\begin{equation}\label{eq:c_gamma}
c_{\gamma-2\eta}:=\sup_{q\in[\gamma-2\eta,1-\gamma+2\eta]} |(v^{-1})'(q)|<\infty.\end{equation}
By the mean value theorem, the inversion error is then bounded linearly by the estimation error:
\[|v^{-1}(\hat q)-v^{-1}(q)|\le c_{\gamma-2\eta}\,|\hat q-q|.\]

This Lipschitz-type control motivates restricting the inversion to the set of \emph{relevant rounds} $M_T^\gamma$, on which every arm's sampling probability lies safely in $[\gamma, 1-\gamma]$. The cardinality $|M_T^\gamma|$ counts the rounds on which the inversion is uniformly well-conditioned. As the demonstrator learns and concentrates its policy on the optimal arm, the sampling probabilities assigned to suboptimal arms decay towards zero, pushing the inversion into the ill-conditioned boundary regime. The restriction to $M_T^\gamma$ is therefore not an artifact of our analysis, but an intrinsic requirement for any stable approximate inversion procedure. 

It suffices to verify the lower bounds
\(q_{a_t,t},q_{\hat a^*,t}\ge\gamma\). Indeed, since
\(a_t\neq\hat a^*\) and the probabilities sum to one, these
conditions automatically imply
\(q_{a_t,t},q_{\widehat a^*,t}\le1-\gamma\).
We then restrict $\gamma\in(0,1/2)$ in~\cref{def:safe_round} otherwise \(q_{a_t,t},q_{\hat a^*,t}\ge\gamma\) could not be simultaneously satisfied for $\gamma>1/2$. 

\section{Generalized LITE sampling}\label[appendix]{app:gen_lite}

While the main text focuses on the Gaussian reward setting, where the posterior uncertainties $\sigma$ are deterministically known to the observer, enabling exact closed-form inversion, extending LITE~\citep{menet2025lite} to arbitrary arm distributions is a central secondary contribution of this paper. Many real-world applications rely on reward distributions that are non-Gaussian: click-through rates are Bernoulli, waiting times are Exponential, and clinical outcomes can follow heavy-tailed distributions. To invert demonstrators operating in these settings, we first formalize a generic sampling strategy that is both no-regret and invertible for arbitrary reward posteriors. The key idea is to characterize the tail behavior of the marginal reward posteriors through a reference CDF $v$.

To this end, let $v$ be a continuous, strictly increasing CDF. The \emph{conditional tail expectation functional} $g_v:(0,1)\to\mathbb R$ is defined by 
\[g_v(p) := \mathbb E[Z | Z > v^{-1}(p)] \cdot \mathbb P[Z > v^{-1}(p)] = \int_p^1 v^{-1}(r) dr\]
Moreover, for $\mathbb E[Z] \in [0, \infty)$, it holds that $g_v(p) \in [0, \infty)\ \forall p$.

The associated \emph{quasi-surprisal} $i_v:(0,1)\to[0,\infty)$ is defined by
\begin{equation*}
    i_v(p) := \frac{1}{2}\left(\frac{g_v(p)}{p}\right)^2.
\end{equation*}
The quasi-surprisal measures how aggressively the tail of $v$ decays: heavy-tailed distributions yield large $i_v$ values for moderate $p$, while light-tailed distributions (such as the Gaussian) yield smaller values. When the reference function $v$ is chosen to match the actual marginal tail decay of the posterior $\theta$, the resulting $v$-LS policy is guaranteed to be no-regret.

The following proposition makes this precise. It shows that if the posterior's Cram\'{e}r rate function dominates the quasi-surprisal induced by $v$, then the expected cumulative regret is sublinear. The quantity $H_v(q) := \sum_a q_a\, i_v(q_a)$ is the \emph{adapted quasi-entropy}, which bounds the total exploration penalty of the randomized policy.
\begin{definition}\label{def:cramer_rate_function}
    The (Cramér) rate function $\Lambda^*$ of a random variable $X: \Omega \to \mathbb R$ is defined as the convex conjugate of the cumulant generating function $\Lambda$, i.e., 
    \begin{equation*}
        \Lambda^*(\alpha) := \sup_{\beta > 0} \alpha \beta - \Lambda(\beta) \text{ where } \Lambda(\beta) := \log \mathbb E[e^{\beta (X-\mathbb E[X])}].
    \end{equation*}
\end{definition}

\begin{restatable}{proposition}{cumregret}\label{prop:cum_regret}
    Let $\theta_t:\Omega \to \mathbb R^K$ be a sequence of random vectors with marginal (Cramér) rate function satisfying $(\Lambda_{a_t}^*)^{-1}(-\ln p) \leq \sigma_{a_t,t} \sqrt{2 i_v(p)}$ for all $p \in (0,1)$. Let $q_t$ be the $v$-LITE estimator at step $t$ and $a_t \sim q_t$. Assume the adapted quasi-entropy is bounded: $H_v(q) := \sum_{a\in[K]} q_a i_v(q_a) \le \sup_{q\in\Delta^K}H_v(q)=: I_v(K) < \infty$, and that the cumulative predictive variance is bounded by a sublinear function $V_T$, i.e., $\mathbb{E}\big[\sum_{t=1}^T \sigma_{a_t,t}^2\big] \le V_T$. Then it holds that
    \begin{align*}
        Reg(T) &= \textstyle \sum_t \mathbb E[ \max_a \theta_{a,t} - \theta_{a_t,t}] \\
        & \textstyle \leq \sum_t \mathbb E_{a \sim q_t}  [\sigma_{a,t} \sqrt{2 i_v(q_{a,t})}] \\
        & \le \sqrt{2 I_v(K) T V_T}.
    \end{align*}
\end{restatable}
\begin{proof}[Proof of Proposition~\ref{prop:cum_regret}]
    The first inequality follows immediately from Lemma~\ref{lem:expected_maximum_upper_bound}:
    \begin{align*}
        \textstyle \sum_{t=1}^T \mathbb E[ \max_a \theta_{a,t} - \theta_{a_t,t}] & \leq \textstyle\sum_{t=1}^T \mathbb E \Big[\mathcal V(q_t) - \sum_{a=1}^K q_{a,t}\mu_{a,t}\Big] = \textstyle \sum_{t=1}^T \mathbb E \Big[\sum_{a=1}^K q_{a,t} \sigma_{a,t} \sqrt{2 i_v(q_{a,t})}\Big].
    \end{align*}
    For the second inequality, we apply the Cauchy-Schwarz inequality first over the arms, then over the time steps:
    \begin{align*}
        \sum_{t=1}^T \sum_{a=1}^K q_{a,t} \sigma_{a,t} \sqrt{2 i_v(q_{a,t})} 
        &\le \sum_{t=1}^T \sqrt{\sum_{a=1}^K 2 q_{a,t} i_v(q_{a,t})} \sqrt{\sum_{a=1}^K q_{a,t} \sigma_{a,t}^2} \\
        &\le \sqrt{2 I_v(K)} \sum_{t=1}^T \sqrt{\sum_{a=1}^K q_{a,t} \sigma_{a,t}^2} \\
        &\le \sqrt{2 I_v(K)} \sqrt{T} \sqrt{\sum_{t=1}^T \sum_{a=1}^K q_{a,t} \sigma_{a,t}^2}.
    \end{align*}
    Taking the expectation over the trajectory and applying Jensen's inequality alongside the assumption on the cumulative variance yields:
    \begin{align*}
        \mathbb{E}\Bigg[\sqrt{2 I_v(K) T \sum_{t=1}^T \sum_{a=1}^K q_{a,t} \sigma_{a,t}^2}\Bigg] \le \sqrt{2 I_v(K) T \mathbb{E}\Bigg[\sum_{t=1}^T \sum_{a=1}^K q_{a,t} \sigma_{a,t}^2\Bigg]} = \sqrt{2 I_v(K) T \mathbb{E}\Bigg[\sum_{t=1}^T \sigma_{a_t,t}^2\Bigg]} \le \sqrt{2 I_v(K) T V_T}.
    \end{align*}
\end{proof}

In practical deployments, including large-scale A/B tests and federated learning systems, updating the posterior after every observation is often computationally prohibitive or impossible due to delayed feedback pipelines. Batched algorithms address this by freezing the sampling policy over windows of $n$ consecutive rounds and updating only at batch boundaries. The following theorem extends the regret analysis of \cref{prop:cum_regret} to this batched setting and provides the theoretical foundation for inverting batch-constrained demonstrators.

\begin{restatable}[Batched LS regret]{theorem}{batchedregret}\label{thm:batchedregret}
    Assume that $v$ satisfies the assumptions of~\cref{prop:cum_regret} and that the adapted quasi-entropy is bounded: $H_v(q) := \sum_{a\in[K]} q_a i_v(q_a) \le \sup_{q\in\Delta^K}H_v(q)=: I_v(K) < \infty$. 
    If we assume Gaussian reward noise with variance $\alpha^2$ and Gaussian priors, then batched $v$-LS with batch size $n$ and horizon $T = nB$ satisfies
\begin{align*}
  Reg(T) \le 2\sqrt{I_v(K) T \frac{n}{C(n,\alpha)} \gamma_B},
\end{align*}
where $C(n,\alpha) = \ln(1+n/\alpha^2)$ and $\gamma_B$ is the maximal information gain over $B$ batches.
\end{restatable}

\begin{remark}
    For $v=\Psi$, we have $I_\Psi(K)\le \log(K)$, and the bound simplifies to $\mathcal{O}(\sqrt{T \gamma_B \frac{n}{C(n,\alpha)} \log K})$. The ratio $n / \ln(1+n/\alpha^2)$ grows almost linearly in $n$, reflecting the structural penalty of delayed posterior updates. Note that the bound recovers the non-batched rate when $n=1$.
\end{remark}
\begin{proof}[Proof of \cref{thm:batchedregret}]
    Since batched $v$-LS updates its posterior only at batch boundaries, the sampling probabilities $q_b$ and posterior standard deviations $\sigma_{a,b}$ remain constant throughout each batch. Let $a_{i,b}$ denote the $i$-th arm pulled in batch $b$. By \cref{lem:expected_maximum_upper_bound}, applied with the posterior at the start of batch $b$, we have
    \begin{align*}
        \mathrm{Reg}(T) = \sum_{b=1}^B\sum_{i=1}^{n} \mathbb{E}[\max_a \theta_{a}-\theta_{a_{i,b}}] \le \sum_{b=1}^B\sum_{i=1}^{n} \mathbb{E}_{a_{i,b}\sim q_b}[\sigma_{a_{i,b},b}\sqrt{2i_v(q_b(a_{i,b}))}].
    \end{align*}
    This is identical in structure to the first inequality in the proof of \cref{prop:cum_regret}. Applying the same double Cauchy--Schwarz argument (first over arms within each batch, then over batches) yields
    \begin{equation*}
        \mathrm{Reg}(T) \le \sqrt{2 I_v(K) T \,\mathbb{E}\Bigg[\sum_{b=1}^B \sum_{i=1}^{n} \sigma_{a_{i,b},b}^2\Bigg]}.
    \end{equation*}
    Under the Gaussian prior and Gaussian noise assumptions, $\sum_{i=1}^{b} \sigma_{a_{i,b},b}^2 = \mathrm{tr}(\Sigma_b(X_{1:b}))$, where $\Sigma_b(X_{1:b})$ is the $n \times n$ posterior covariance matrix of the arms pulled in batch $b$ under the realized actions. Our batched generalization of the Elliptical Potential Lemma (\cref{lem:4_batched}) bounds this for any realized action sequence $X_{1:B}$ by the maximum information gain over fixed designs $\gamma_B$, giving
    \begin{align*}
        \mathbb{E}\Bigg[\sum_{b=1}^B \mathrm{tr}(\Sigma_b(X_{1:b}))\Bigg] \le \frac{2n}{C(n,\alpha)} \gamma_B,
    \end{align*}
    where $C(n,\alpha) = \ln(1+n/\alpha^2)$. Substituting completes the proof.
\end{proof}

\section{Proofs of Section \ref{sec:approx_ilite}}
\label[appendix]{app:proofs_ilite}

\gapeps*
\begin{proof}
    \begin{align*}
        & |\hat\Delta_{a_t,t,}-\Delta^{LS}_{a_t,t,}|:= |\hat z_{\hat a^*,t}\sigma_{\hat a^*,t}-\hat z_{a_t,t}\sigma_{a_t,t}-(z_{\hat a^*,t}\sigma_{\hat a^*,t}-z_{a_t,t}\sigma_{a_t,t})|\\
        &\le |\hat z_{\hat a^*,t}\sigma_{\hat a^*,t}-z_{\hat a^*,t}\sigma_{\hat a^*,t}|+|z_{a_t,t}\sigma_{a_t,t}-\hat z_{a_t,t}\sigma_{a_t,t}|\\
        &\le |\sigma_{\hat a^*,t}||\hat z_{\hat a^*,t}-z_{\hat a^*,t}|+|\sigma_{a_t,t}||\hat z_{a_t,t}-z_{a_t,t}|
    \end{align*}
    Observe that, for $t\in\widehat M^{\rho}_T$
    \begin{align*}
        |\hat z_{a_t,t}-z_{a_t,t}|:=|v^{-1}(\hat q_{a_t,t})-v^{-1}(q_{a_t,t})|\le c_{\rho-\eta}|\hat q_{a_t,t}-q_{a_t,t}|\le c_{\rho-\eta} \eta,
    \end{align*}
    since $v^{-1}:[\rho-\eta, 1-\rho+\eta] \to \mathbb R$ is Lipschitz, with Lipschitz constant $ c_{\rho-\eta}$. The same inequality holds for $\hat a^*$.
    Combining the inequalities concludes the proof.
\end{proof}

\approxlite*
\begin{proof}
Using the fact that $M_T^\gamma\subseteq\widehat M_T^{\gamma-\eta}$ and~\cref{lem:gaps_eps} with $\rho=\gamma-\eta$, we can write
\begin{align*}
    &\mathcal{E}_T^\gamma\le\frac{1}{T}\sum_{t=1}^T \mathds{1}\{t\in \widehat M_T^{\gamma-\eta}\}| \hat \Delta_{a_t,t} - \Delta_{a_t,t}^{LS}|\\
    &\le \frac{c_{\gamma-2\eta} \eta}{T}\sum_{t=1}^T \mathds{1}\{t\in\widehat M_T^{\gamma-\eta}\}(|\sigma_{a
    _t,t}|+|\sigma_{\hat{a}^*,t}|)
\end{align*}

Using the Cauchy-Schwarz inequality for the first sum, we get: 
\begin{align}
&\frac{c_{\gamma-2\eta} \eta}{T}\sum_{t=1}^T \mathds{1}\{t\in \widehat M_T^{\gamma-\eta}\}(|\sigma_{a
    _t,t}|+|\sigma_{\hat{a}^*,t}|)\\
    & \le \frac{c_{\gamma-2\eta} \eta}{T} \Big( \sqrt{\sum_{t=1}^T \sigma_{a_t,t}^2}\times\sqrt{T}+\sum_{t=1}^T \mathds{1}\{t\in \widehat M_T^{\gamma-\eta}\}\sigma_{\hat a^*,t}\Big)\\
    &\le\frac{c_{\gamma-2\eta} \eta}{T} \Big( \sqrt{T} \sqrt{\frac{C_{\alpha}\gamma_T}{2}}+\sum_{t=1}^T \sigma_{\hat a^*,t}\Big).\label{eq:tau_star_sum}
\end{align}
Note that the last inequality applied~\cref{lem:4}. It remains to bound 
\[
\sum_{t=1}^T \sigma_{\hat a^*,t}.
\]
First, observe that: 
\[\sigma_{\hat a^*,t}\le\sqrt{\Big(\tfrac{1}{\sigma_{\hat a^*,0}^2}+\tfrac{N_{\hat a^*(t-1)}}{\alpha^2}\Big)^{-1}}\le \frac{\alpha}{\sqrt{N_{\hat a^*}(t-1)}}.\]
We know that, 
\begin{equation}\label{eq:na_larger}
N_{\hat a^*}(t-1)=(t-1)-\sum_{a\neq\hat a^*}N_a(t-1)\ge (t-1)-\sum_{a\neq\hat a^*}N_a(T),
\end{equation}
here we should be careful because $(t-1)-\sum_{a\neq\hat a^*}N_a(T)$ could be negative. Therefore, denote $t^*:=\sum_{a\neq\hat a^*}N_a(T)<T$, hence:
\begin{align}
    &\sum_{t=1}^T \sigma_{\hat a^*,t}= \sum_{t=1}^{t^*+1} \sigma_{\hat a^*,t}+\sum_{t=t^*+2}^T \sigma_{\hat a^*,t}\\
    &\le t^*+1+\sum_{t=t^*+2}^{T}\frac{\alpha}{\sqrt{N_{\hat a^*}(t-1)}}\\
    & \le t^*+1 + \alpha\Bigg(\sum_{t=t^*+2}^{T}\frac{1}{\sqrt{t-1-t^*}}\Bigg)\le t^* +1 + 2\alpha\sqrt{T}
\end{align}
\begin{remark}
    Since $\hat a^*$ is a most-pulled arm and $T\ge1$,
$N_{\hat a^*}(T)\ge1$. Hence
$t^*=T-N_{\widehat a^*}(T)\le T-1$, so
$t^*+1\le T$.
\end{remark}

Where we upper bounded $\sigma_{\hat a^*,t}\le \sigma_{\hat a^*,0}\le1$ in the first sum, and we use~\cref{eq:na_larger} in the second one together with an integral comparison. Plugging this into~\cref{eq:tau_star_sum} we get: 
\begin{equation}
    \frac{1}{T}\sum_{t=1}^T \mathds{1}\{t\in M_{T}^\gamma\}| \hat \Delta_{a_t,t} - \Delta_{a_t,t}^{LS}|\le\frac{c_{\gamma-2\eta} \eta}{T} \Big( \sqrt{T} \sqrt{\frac{C_{\alpha}\gamma_T}{2}}+t^*+1+2\alpha\sqrt{T}\Big),
\end{equation}
where $C_{\alpha}:=\frac{4}{\ln{(1+\alpha^{-2})}}$.
Finally, observing that $\frac{2}{\ln(1+\alpha^{-2})} \le 2\alpha^2 + 1$ (since $\ln(1+x) \ge \frac{2x}{x+2}$ for $x>0$) and $\sqrt{2\alpha^2+1} \le \alpha\sqrt{2}+1$, concludes the proof.
\end{proof}

\estimatedPoM*
\begin{proof}
Fix an arm $a\in[K]$ and a batch $b\in[B]$. Conditionally on the
history $\mathcal F_{(b-1)n}$ available before batch $b$, define
\[
    X_{b,j}^a:=\mathbf{1}\{a_{b,j}=a\},
    \qquad j\in[n].
\]
Because the sampling distribution remains fixed within batch $b$, the
variables $X_{b,1}^a,\ldots,X_{b,n}^a$ are conditionally independent
Bernoulli random variables with conditional mean $q_{a,b}$. Moreover,
\[
    \hat q_{a,b}
    =
    \frac{1}{n}\sum_{j=1}^{n}X_{b,j}^a.
\]
Therefore, Hoeffding's inequality gives, for every $x>0$,
\begin{equation}
\begin{aligned}
    \mathbb{P}\left(
        \left|
            \hat q_{a,b}-q_{a,b}
        \right|>x
        \,\middle|\,
        \mathcal F_{b-1}
    \right)
    \leq 2\exp(-2nx^2).
\end{aligned}
\end{equation}
Taking expectations over $\mathcal F_{b-1}$ yields
\begin{equation}
    \mathbb{P}\left(
        \left|
            \hat q_{a,b}-q_{a,b}
        \right|>x
    \right)
    \leq 2\exp(-2nx^2).
\end{equation}
Applying a union bound over the $KB$ arm-batch pairs gives
\begin{equation}
\begin{aligned}
    &\mathbb{P}\left(
        \max_{b\in[B]}
        \max_{a\in[K]}
        \left|
            \hat q_{a,b}-q_{a,b}
        \right|>x
    \right) \\
    &\hspace{4em}
    \leq
    2KB\exp(-2nx^2).
\end{aligned}
\end{equation}
Setting
\[
    x
    =
    \sqrt{\frac{\log(2KB/\delta)}{2n}}
    =
    \eta_{n,\delta}
\]
concludes the proof.
\end{proof}

\batchederror*
\begin{proof}
On the event \(\mathcal G_{n,\delta}\), Lemma~\ref{lem:gaps_eps}
gives
    \begin{align*}
        &\mathcal{E}_T^\gamma\le\frac{1}{T}\sum_{t=1}^T \mathds{1}\{t\in \widehat M_T^{\gamma-{\eta_{n,\delta}}}\}| \hat \Delta_{a_t,t} - \Delta_{a_t,t}^\mathcal{A}| =\\
        &=\frac{1}{T}\sum_{b=1}^B\sum_{t=n(b-1)+1}^{nb}\mathds{1}\{t\in \widehat M_T^{\gamma-{\eta_{n,\delta}}}\}| \hat \Delta_{a_t,t} - \Delta_{a_t,t}^\mathcal{A}| \\
        &\le \frac{c_{\gamma-2\eta_{n,\delta}} \eta_{n,\delta}}{T}\sum_{b=1}^B\sum_{t=n(b-1)+1}^{nb}\mathds{1}\{t\in \widehat M_T^{\gamma-{\eta_{n,\delta}}}\} (|\sigma_{a_t,t}|+|\sigma_{\hat a^*,t}|).
    \end{align*}
    Using Cauchy-Schwarz inequality we get 
    \begin{align}\label{eq:bound}
        &\frac{c_{\gamma-2\eta_{n,\delta}} \eta_{n,\delta}}{T}\sum_{b=1}^B\sum_{t=n(b-1)+1}^{nb}\mathds{1}\{t\in \widehat M_T^{\gamma-{\eta_{n,\delta}}}\} (|\sigma_{a_t,t}|+|\sigma_{\hat a^*,t}|)\le\\
        &\frac{c_{\gamma-2\eta_{n,\delta}} \eta_{n,\delta}}{T}\left(
        \sqrt{\sum_{b=1}^B\sum_{t=n(b-1)+1}^{nb}|\sigma_{a_t,t}|^2}\times\sqrt{nB}
        +
        \sum_{b=1}^B\sum_{t=n(b-1)+1}^{nb}|\sigma_{\hat a^*,t}|
        \right)
    \end{align}
    Since the posterior is constant within each batch, $\sum_{t=n(b-1)+1}^{nb}\sigma_{a_t,t}^2=tr(\Sigma_b)$, with $\Sigma_b$ covariance matrix as defined in~\cref{lem:4_batched} for any sequence of $a_t$. Lemma~\ref{lem:4_batched} holds for every deterministic batched design
and can therefore be applied pathwise to the realized action sequence:
    \begin{equation}
        \sqrt{\sum_{b=1}^B\sum_{t=n(b-1)+1}^{nb}|\sigma_{a_t,t}|^2}\times\sqrt{nB} = \sqrt{nB}
        \sqrt{\sum_{b=1}^B tr(\Sigma_b)}\le \sqrt{nB} \sqrt{\frac{2n\gamma_B}{C(n,\alpha)}},
    \end{equation}
where $C(n,\alpha):=\ln(1+n/\alpha^2).$
Note that, $\frac{2n}{\ln(1+n/\alpha^2)}\le2\alpha^2+n$ and $\sqrt{2\alpha^2+n}\le \alpha\sqrt{2}+\sqrt{n}$, hence: 
\begin{equation}\label{eq:bound_with_lem44}
\sqrt{nB} \sqrt{\frac{2n\gamma_B}{C(n,\alpha)}}\le\sqrt{nB}(\alpha\sqrt{2}+\sqrt{n})\sqrt{\gamma_B}. 
\end{equation}
It remains to bound: 
\begin{equation}\label{eq:bound_ahat}
    \sum_{b=1}^B\sum_{t=n(b-1)+1}^{nb}|\sigma_{\hat a^*,t}| = \sum_{b=1}^B n \sigma_{\hat a^*,b}
\end{equation}
since $\sigma_{\hat a^*,t}$ gets updated only at the beginning of each batch, hence it stays constant for $n$ steps (batch size). 
Here we proceed analogously to the proof of~\cref{thm:approx_lite}. 
Write
\[
t^*=\sum_{a\neq\hat a^*}
N_a(T)=\ell n+r,
\qquad
\ell:=\left\lfloor\frac{t^*}{n}\right\rfloor,
\qquad
r:=t^*\bmod n\in\{0,\ldots,n-1\}.
\]
At the beginning of batch $b$,
\[
N_{\hat a^*}(n(b-1))
=
n(b-1)
-
\sum_{a\neq\hat a^*}
N_a(n(b-1))
\ge
n(b-1)-t^*,
\]
the first \(\ell+1\) batches contribute at most \(n(\ell+1)\), using
\(\sigma_{\hat a^*,t}\le1\).

If \(\ell+2\le B\), we treat batch \(\ell+2\) separately and use
\[
n\sigma_{\hat a^*,n(\ell+2)}
\le
n\min\left\{
1,\frac{\alpha}{\sqrt{n-r}}
\right\}
\le
n\min\{1,\alpha\}=:nB_\alpha
\]
For the remaining batches, the summand is decreasing, so an integral
comparison gives
\begin{align}
\sum_{b=\ell+3}^B
\frac{n\alpha}
{\sqrt{n(b-1)-t^*}}
&\le
n\alpha
\int_{\ell+2}^{B}
\frac{dx}{\sqrt{n(x-1)-t^*}}\\
&=
2\alpha
\left(
\sqrt{n(B-1)-t^*}
-
\sqrt{n-r}
\right)\\
&\le
2\alpha\sqrt{T}.
\end{align}
Consequently,
\begin{align}
\sum_{b=1}^B n \sigma_{\hat a^*,nb}
&\le
n(\ell+1)
+
nB_\alpha
+
2\alpha\sqrt T\\
&=
t^*-r+n
+
nB_\alpha
+
2\alpha\sqrt T\\
&\le
t^*
+
n\bigl(1+B_\alpha\bigr)
+
2\alpha\sqrt T.
\label{eq:batched_reference_bound}
\end{align}
If \(\ell+2>B\), the corresponding tail sum is empty and the same
upper bound continues to hold.

Plugging~\cref{eq:batched_reference_bound} and ~\cref{eq:bound_with_lem44} into~\cref{eq:bound} we get: 
\begin{align}
    &\frac{c_{\gamma-2\eta_{n,\delta}} \eta_{n,\delta}}{T}\left(
        \sqrt{\sum_{b=1}^B\sum_{t=n(b-1)+1}^{nb}|\sigma_{a_t,t}|^2}\times\sqrt{nB}
        +
        \sum_{b=1}^B\sum_{t=n(b-1)+1}^{nb}|\sigma_{\hat a^*,t}|
        \right)\\
    &\le\frac{c_{\gamma-2\eta_{n,\delta}} \eta_{n,\delta}}{T}\left(
        \sqrt{nB}(\alpha\sqrt{2}+\sqrt{n})\sqrt{\gamma_B}
        +
        t^*+n\bigl(1+\min\{1,\alpha\}\bigr)+2\alpha\sqrt{T})
        \right)
\end{align}
Since \(\mathbb P(\mathcal G_{n,\delta})\ge1-\delta\) by
Lemma~\ref{lem:estimate_p}, the result holds with probability at least
\(1-\delta\).
\end{proof}

\epsbound*
\begin{proof}
\begin{align}
    \mathcal{E}_T^\gamma&\le \frac{1}{T}\sum_{t=1}^T \mathds{1}\{t\in \widehat M_T^{\gamma-\varepsilon}\}| \hat \Delta_{a_t,t} - \Delta_{a_t,t}^{\mathcal{A}}| \\
    &\le \frac{1}{T}\sum_{t=1}^T \mathds{1}\{t\in \widehat M_T^{\gamma-\varepsilon}\}\Big(| \hat \Delta_{a_t,t} - \Delta_{a_t,t}^{LS}|+ | \Delta_{a_t,t}^{LS}- \Delta_{a_t,t}^{\mathcal{A}}|\Big)
\end{align}
We first bound $\sum_{t=1}^T \mathds{1}\{t\in \widehat M_T^{\gamma-\varepsilon}\}| \hat \Delta_{a_t,t} - \Delta_{a_t,t}^{LS}|$. The proof is similar to the one of~\cref{thm:approx_lite}. We apply~\cref{lem:gaps_eps}, by observing that $|p_{a_t,t}-q_{a_t,t}|\le\varepsilon$ and replacing $\eta$ with $\varepsilon$, hence:
\[\sum_{t=1}^T \mathds{1}\{t\in \widehat M_T^{\gamma-\varepsilon}\}| \hat \Delta_{a_t,t} - \Delta_{a_t,t}^{LS}|\le c_{\gamma-2\varepsilon}\varepsilon\sum_{t=1}^T \mathds{1}\{t\in \widehat M_T^{\gamma-\varepsilon}\}(|\sigma_{a
    _t,t}|+|\sigma_{\hat{a}^*,t}|)\] 
Then, the proof is analogous to the one of~\cref{thm:approx_lite} and 
\[ \frac{1}{T}\sum_{t=1}^T \mathds{1}\{t\in \widehat M_T^{\gamma-\varepsilon}\}| \hat \Delta_{a_t,t} - \Delta_{a_t,t}^{LS}|\le c_{\gamma-2\varepsilon} \varepsilon \Big(\frac{(\alpha\sqrt{2}+1)\sqrt{\gamma_T}+2\alpha}{\sqrt{T}}+\frac{t^*+1}{T}\Big). \]
Then we bound,  $\sum_{t=1}^T \mathds{1}\{t\in\widehat M_T^{\gamma-\varepsilon}\}|\Delta_{a_t,t}^{LS}- \Delta_{a_t,t}^{\mathcal{A}}|$, here we know by~\cref{ass:same_gaps} that $|\Delta_{a_t,t}^{LS}- \Delta_{a_t,t}^{\mathcal{A}}|\le\tilde\varepsilon$ hence:
\[\sum_{t=1}^T \mathds{1}\{t\in \widehat M_T^{\gamma-\varepsilon}\}|\Delta_{a_t,t}^{LS}- \Delta_{a_t,t}^{\mathcal{A}}|\le\tilde\varepsilon T.\]
Combining the bounds concludes the proof.
\end{proof}

\subsection{Technical Lemmas}

\begin{lemma}[Variational form of KL-divergence]\label{lem:variational_form_of_KL_divergence}
    Fix two probability distributions \(p:\Sigma \to [0,1]\) and \(q:\Sigma \to [0,1]\) over the measurable space \((\Omega, \Sigma)\) such that \(p\) is absolutely continuous with respect to \(q\) (\(p \ll q\)). Then
    \begin{equation*}
        D_{KL}[p||q] = \sup_{X}\{\mathbb E_{p} [X] - \ln \mathbb E_{q} [\exp X]\},
    \end{equation*}
    where the supremum is taken over all measurable \(X:\Omega \to \mathbb R\) such that \(\mathbb E_{p} [X]\) and \(\mathbb E_{q} [\exp X]\) are well-defined.
\end{lemma}

\begin{proof}
    The proof generalizes Theorem 3.2 in \cite{gray2011entropy} from discrete spaces to arbitrary probability spaces.
    Since \(p \ll q\), there exists a Radon-Nikodym derivative \(\tfrac{d p}{dq}(\omega)\), i.e., $\textstyle p(\mathcal A) = \int_{\mathcal A} \tfrac{d p}{dq}(\omega)\, dq(\omega)$ for a function $\tfrac{d p}{dq}(\omega)$ uniquely defined $q$-a.e.
    Now, let \(X:\Omega \to \mathbb R\) be any random variable for which \(\mathbb E_{p} [X]\) and \(\mathbb E_{q} [\exp X]\) are well-defined. Then 
    \begin{align*}
        D_{KL}[p||q] - (\mathbb E_p [X] - \ln(\mathbb E_q [\exp X])) & = \mathbb E_p [\ln \frac{dp}{dq}(\omega) ] - \mathbb E_p [\ln\frac{\exp X}{\mathbb E_q [\exp X]}]\\
        = \mathbb E_p [\ln \left(\frac{d p}{d q} (\omega) \frac{\mathbb E_q [\exp X]}{\exp (X)}\right)] & = \mathbb E_p [\ln \frac{d p}{d \lambda}] = D_{KL}[p || \lambda ]\geq 0,
    \end{align*}
    where $\textstyle \lambda(\mathcal A) := \int_{\mathcal A} \exp (X(\omega)) / \mathbb E_q [\exp X]\, dq(\omega)$ is the tilted measure, and we used $\tfrac{\mathbb E_q [\exp X]}{\exp (X)} = \tfrac{dq}{d\lambda}$.
\end{proof}

\begin{lemma}[Conditioned KL-divergence]\label{lem:conditioned_KL_divergence}
    Consider the probability space \((\Omega, \Sigma, \mathbb P)\) and an event \(\mathcal B \in \Sigma\) of non-zero probability, i.e. \(\mathbb P[\mathcal B] > 0\). Then
    \begin{equation*}
        D_{KL}[\mathbb P[\ \cdot\ | \mathcal B]\ ||\ \mathbb P] = - \ln \mathbb P[\mathcal{B}].
    \end{equation*}
\end{lemma}

\begin{proof} According to the definition of conditional expectation it holds
\begin{equation*}
    \mathbb P[\mathcal A\ | \mathcal B] = \frac{\mathbb P[\mathcal A \cap \mathcal B]}{\mathbb P[\mathcal B]} = \int_{\mathcal A} \frac{\mathds{1}_{\omega \in \mathcal B}}{\mathbb P[\mathcal B]}\ d\mathbb P(\omega) \qquad \forall \mathcal A \in \Sigma,
\end{equation*}
where we recognize absolute continuity \(\mathbb P[\ \cdot\ | \mathcal B] \ll \mathbb P\) and identify the Radon-Nikodym derivative $\tfrac{d \mathbb P[\ \cdot\ |\mathcal B]}{d \mathbb P}(\omega) = \tfrac{\mathds{1}_{\omega \in \mathcal B}}{\mathbb P[\mathcal B]}$. Plugging the derivative into the definition of Kullback-Leibler divergence results in the desired expression:
\begin{align*}
    D_{KL}[\mathbb P[\ \cdot\ | \mathcal B]\ ||\ \mathbb P] := & \int_\Omega \ln \frac{d \mathbb P[\ \cdot\  |\mathcal B]}{d \mathbb P} d \mathbb P[\ \cdot\ | \mathcal B] = \int_\Omega \ln (\frac{d \mathbb P[\ \cdot\  |\mathcal B]}{d \mathbb P}) \frac{d \mathbb P[\ \cdot\  |\mathcal B]}{d \mathbb P} d \mathbb P\\
    & = \int_\Omega \ln (\frac{\mathds{1}_{\omega \in \mathcal B}}{\mathbb P[\mathcal B]}) \frac{\mathds{1}_{\omega \in \mathcal B}}{\mathbb P[\mathcal B]} d \mathbb P = - \ln \mathbb P[\mathcal B].
\end{align*}
\end{proof}

\begin{lemma}[Upper bound on conditional expectation]\label{lem:upper_bound_conditional_expectation}

    Let \(X:\Omega \to \mathbb R\) be a random variable on \((\Omega, \Sigma, \mathbb P)\) such that the (restricted) cumulant generating function $\Lambda: \mathbb R^+ \to [0,\infty)\quad \beta \mapsto \ln \mathbb E [\exp(\beta(X-\mathbb E [X]))]$ exists. Assume further that \(\mathbb P[\mathcal B] > 0\), ensuring \(\mathbb P[\cdot | \mathcal B]\) is well-defined. Then with \(\Lambda^*\) the Cramer rate function of $X$ (see Definition~\ref{def:cramer_rate_function}) and $(\Lambda^*)^{-1}(y) := \inf\{s : \Lambda^*(s) \ge y\}$ its generalized inverse, it holds that
    \begin{equation*}
        \mathbb E[X | \mathcal B] \leq \mathbb E [X] + (\Lambda^*)^{-1} (- \ln \mathbb P[\mathcal B]).
    \end{equation*}
    
\end{lemma}
\begin{proof}
    The proof adapts Lemma 11 of \citet{tarbouriech2023probabilistic}.
    We apply Lemma~\ref{lem:variational_form_of_KL_divergence} with \(p(\mathcal E) = \mathbb P[\mathcal E| \mathcal B]\) and \(q(\mathcal E) = \mathbb P[\mathcal E]\), and restrict the supremum over just the random variables \(\{\lambda(X-\mathbb E X)\}_{\lambda \in \mathbb R_+}\). This gives
    \begin{align*}
        D_{KL}[\mathbb P[\cdot| \mathcal B] || \mathbb P] & \geq \sup_{\lambda \in \mathbb R_+} \{ \lambda \mathbb E [X-\mathbb E[X] \ |\ \mathcal B] - \ln \mathbb E [\exp(\lambda(X-\mathbb E [X]))]\}\\
        & = \sup \{ \lambda (\mathbb E [X | \mathcal B] - \mathbb E[X]) - \Lambda(\lambda) : \lambda \in \mathbb R_+\}\\
        & = \Lambda^*(\mathbb E [X | \mathcal B] - \mathbb E[X])
    \end{align*}
    
    Furthermore, since the free parameter \(\lambda\) is restricted to positive values \(\mathbb R^+\), the one-sided Cramér transform \(\Lambda^*\) is strictly monotonically increasing on its strictly positive upper-tail region. Thus, applying the generalized inverse $(\Lambda^*)^{-1}(y) := \inf\{s : \Lambda^*(s) \ge y\}$ yields
    \begin{equation*}
        (\Lambda^*)^{-1}(D_{KL}[\mathbb P[\cdot| \mathcal B]\ ||\ \mathbb P]) \geq \mathbb E [X | \mathcal B] - \mathbb E[X].
    \end{equation*}
    As a final step, we use Lemma~\ref{lem:conditioned_KL_divergence}, which states that $D_{KL}[\mathbb P[\cdot | \mathcal B]\ ||\ \mathbb P] = - \ln \mathbb P[\mathcal B]$.
\end{proof}

\begin{lemma}\label{lem:expected_maximum_upper_bound}
    Let $\theta:\Omega \to \mathbb R^K$ be a random vector with marginal (Cramér) rate function satisfying $(\Lambda_a^*)^{-1}(-\ln p) \leq \sigma_{a} \sqrt{2 i_v(p)}$ for all $p \in (0,1)$. Define the variational value $\mathcal V(p) := \sum_{a} p_a (\mu_a + \sigma_a \sqrt{2 i_v(p_a)})$, the probability of maximality $p_a^* := \mathbb P[a^* = a]>0$, and let $q$ be the $v$-LITE estimator. Then
    \begin{equation*}
        \mathbb E[\theta_{a^*}] \leq \mathcal V(p^*) \leq \mathcal V(q).
    \end{equation*}
\end{lemma}
\begin{proof}
    First, note that an application of Lemma~\ref{lem:upper_bound_conditional_expectation} gives
    \begin{gather*}
        \mathbb E[\theta_{a^*} | a^* = a] \leq \mu_a + (\Lambda_a^*)^{-1}(- \ln \mathbb P[a^* = a])
        \implies
        \textstyle \mathbb E[\theta_{a^*}] = \sum_a p_a^*\ \mathbb E[\theta_{a^*} | a^* = a] \leq \sum_a p_a^* ( \mu_a + (\Lambda_a^*)^{-1}(- \ln p_a^*)).
    \end{gather*}
    By substituting $p = p_a^*$ directly into our assumed upper bound on the inverse Cramér transform, we obtain $(\Lambda_a^*)^{-1}(-\ln p_a^*) \le \sigma_a \sqrt{2 i_v(p_a^*)}$. Plugging this in, we obtain $\mathbb E[\theta_{a^*}] \leq \mathcal V(p^*)$. Finally, we note that $\textstyle \mathcal V(p^*) \leq \max_{p \in \Delta^{K-1}} \mathcal V(p) = \mathcal V(q)$, where the last equality follows from the characterization of the $v$-LITE estimator as the maximizer of $\mathcal V$ over the simplex~\citep{menet2025lite}.
\end{proof}

\begin{lemma}[Restatement of Lemma 4 in~\citet{menet2025thompsonsamplingfinetuningllms}]\label{lem:4}
   Let $\theta \sim \mathcal{N}(\mu_0, \Sigma_0)$ be a multivariate Gaussian prior over $K$ arms, such that the marginal prior variances satisfy $\sigma_{a,0}^2 \le 1$ for all $a \in [K]$.
    Let $L_T:=(a_t)_{t=1}^T$ be a sequence of arm pulls with observations $X_{a_t}=\theta_{a_t}+\varepsilon_{a_t}$. The independent observation noise is distributed as $\varepsilon_{a_t}\sim\mathcal{N}(0,\alpha^2)$ at each step. Then the maximal information gain $\gamma_T$ upper bounds the aggregated predictive marginal variances at the locations (i.e., $\sigma_{a_t,t}^2$),
    \begin{equation*}
        \gamma_T:=\max_{L_T} I(X_{L_T};\theta)\ge \frac{2}{C_{\alpha}}\sum_{t=1}^T \sigma_{a_t,t}^2,
    \end{equation*}
    where $1/C_\alpha:=\frac{\ln(1+\alpha^{-2})}{4}$.
\end{lemma}
\begin{proof}
    We restate the proof from \citet{menet2025thompsonsamplingfinetuningllms}. The right-hand side is well-defined because the predictive variances $\sigma_{a_t,t}^2$ depend only on the observation locations $L_T$, not on the observed values.
    Next, note that $X_{L_T} | \theta$ is a multivariate normal with independent components of variance $\alpha^2$, thus
    \begin{align*}
        I(X_{L_T}; \theta) & = H[X_{L_T}] - H[X_{L_T} | \theta]\\
        & \textstyle = H[X_{L_T}] - \tfrac{1}{2} \sum_{t=1}^T \ln (2 \pi e \alpha^2).
    \end{align*}
    Furthermore, one may decompose 
    \begin{align*}
        H[X_{L_t}] & = H[X_{L_{t-1}}] + H[X_{a_t}|X_{L_{t-1}}]\\
        & = H[X_{L_{t-1}}] + \tfrac{1}{2}\ln(2\pi e (\alpha^2 + \sigma_{a_t,t}^2)),
    \end{align*}
    using that $X_{a_t} | X_{L_{t-1}}$ is Gaussian with variance $\alpha^2 + \sigma_{a_t,t}^2$. Recursively expanding then results in 
    $$\textstyle I(X_{L_T}; \theta) = \frac{1}{2} \sum_{t=1}^T \ln(1+\alpha^{-2} \sigma_{a_t,t}^2).$$
    Finally, by assumption $\sigma_{a_t,t}^2 \in [0,1]$, allowing us to lower bound each summand:
    \begin{align*}
        \sigma_{a_t,t}^2 \leq \tfrac{1}{2}\ln(1+\alpha^{-2}\sigma_{a_t,t}^2) \underbrace{\tfrac{4}{\ln(1+\alpha^{-2})}}_{=:C_{\alpha}}.
    \end{align*}
    Indeed, the inequality is tight for $\sigma_{a_t,t}^2 \in \{0,1\}$ and holds in-between due to concavity of $$\sigma_{a_t,t}^2 \mapsto \ln(1+\alpha^{-2} \sigma_{a_t,t}^2) - \sigma_{a_t,t}^2 \ln(1+\alpha^{-2}).$$
\end{proof}

We generalize the previous lemma to the case of a batched algorithm.
\begin{lemma}[Batched generalization of \cref{lem:4}]
\label{lem:4_batched}
Let \(\theta\sim\mathcal N(\mu_0,\Sigma_0)\) be a multivariate Gaussian prior over
\(K\) arms, with marginal prior variances
\(\sigma_{a,0}^2\le 1\) for every \(a\in[K]\).
Split the \(T=nB\) rounds into \(B\) batches of size \(n\).

Fix a deterministic batched action sequence
\[
x_{1:B}=(x_1,\ldots,x_B),
\qquad
x_b=(a_{b,1},\ldots,a_{b,n})\in[K]^n.
\]
Under this fixed design, define
\[
Y_{b,i}^{x}
=
\theta_{a_{b,i}}+\varepsilon_{b,i},
\qquad
\varepsilon_{b,i}\overset{\mathrm{i.i.d.}}{\sim}
\mathcal N(0,\alpha^2),
\]
and let
\[
Y_b^x=(Y_{b,1}^x,\ldots,Y_{b,n}^x),
\qquad
\theta_b^x
=
(\theta_{a_{b,1}},\ldots,\theta_{a_{b,n}}).
\]
Define the posterior covariance matrix at the beginning of batch \(b\) by
\[
\Sigma_b(x_{1:b})
:=
\operatorname{Cov}
\left(
\theta_b^x
\,\middle|\,
Y_{1:b-1}^x
\right).
\]
Let
\[
\gamma_B
:=
\max_{x_{1:B}}
I\left(Y_{1:B}^x;\theta\right)
\]
be the maximal information gain over \(B\) fixed batch designs. Then,
for every deterministic design \(x_{1:B}\),
\[
\gamma_B
\ge
\frac{C(n,\alpha)}{2n}
\sum_{b=1}^B
\operatorname{tr}\!\left(\Sigma_b(x_{1:b})\right),
\]
where
\[
C(n,\alpha)
:=
\ln\left(1+\frac{n}{\alpha^2}\right).
\]
\end{lemma}

\begin{proof}
Fix an arbitrary deterministic design \(x_{1:B}\). By the chain rule
for mutual information,
\begin{align}
I\left(Y_{1:B}^x;\theta\right)
&=
\sum_{b=1}^B
I\left(
Y_b^x;\theta
\,\middle|\,
Y_{1:b-1}^x
\right) \nonumber\\
&=
\frac12
\sum_{b=1}^B
\ln\det\left(
I_n+\alpha^{-2}\Sigma_b(x_{1:b})
\right).
\label{eq:IX}
\end{align}
The second equality follows because, conditionally on
\(Y_{1:b-1}^x\), the vector \(Y_b^x\) is Gaussian with covariance
\[
\Sigma_b(x_{1:b})+\alpha^2I_n,
\]
whereas, conditionally on both \(\theta\) and \(Y_{1:b-1}^x\), its
covariance is \(\alpha^2I_n\).

Since \(\Sigma_b(x_{1:b})\) is positive semidefinite,
\[
\det\left(
I_n+\alpha^{-2}\Sigma_b(x_{1:b})
\right)
\ge
1+
\frac{\operatorname{tr}(\Sigma_b(x_{1:b}))}{\alpha^2}.
\]
Moreover, posterior marginal variances do not exceed their corresponding
prior variances, and hence
\[
0
\le
\operatorname{tr}\!\left(\Sigma_b(x_{1:b})\right)
\le n.
\]
By concavity of \(z\mapsto\ln(1+z)\) on
\([0,n/\alpha^2]\),
\begin{align}
\ln\det\left(
I_n+\alpha^{-2}\Sigma_b(x_{1:b})
\right)
&\ge
\ln\left(
1+
\frac{\operatorname{tr}(\Sigma_b(x_{1:b}))}{\alpha^2}
\right)\\
&\ge
\frac{\operatorname{tr}(\Sigma_b(x_{1:b}))}{n}
\ln\left(1+\frac{n}{\alpha^2}\right).
\end{align}
Substituting this inequality into \cref{eq:IX} gives
\[
I\left(Y_{1:B}^x;\theta\right)
\ge
\frac{C(n,\alpha)}{2n}
\sum_{b=1}^B
\operatorname{tr}\!\left(\Sigma_b(x_{1:b})\right).
\]
Finally, since
\[
\gamma_B
=
\max_{x_{1:B}}I\left(Y_{1:B}^x;\theta\right),
\]
we have
\[
\gamma_B
\ge
I\left(Y_{1:B}^x;\theta\right)
\]
for every fixed design \(x_{1:B}\), which proves the result.
\end{proof}

\section{Additional Experiments}
\label[appendix]{app:experiments}

\subsection{Role of relevance threshold}
\paragraph{Inversion for $v=\Phi$.}
We evaluate the effect of estimating the sampling probabilities in three
Gaussian bandit environments: $K=2$ with
$\mu=(1/2,1)$, $K=4$ with
$\mu=(0,1/3,2/3,1)$, and $K=8$ with
\[
    \mu_a=1-0.1(a-1),\qquad a\in[K].
\]
For every trajectory and round, we generate a sequence of $500$ auxiliary
categorical samples from the exact $\Phi$-LS policy and construct
$\widehat q_t$ from nested prefixes of size
$m\in\{100,200,500\}$. These samples neither update the learner nor affect
its trajectory. We measure the realized uniform estimation error
\[
    \eta=\max_{t\leq T}\max_{a\in[K]}
    \left|\widehat q_{a,t}-q_{a,t}\right|.
\]
Following the formulation of~\cref{thm:approx_lite},
inversion is performed on the enlarged practical relevance set
$\widehat{\mathcal M}_T^{\gamma-\eta}$, and quantities involving the inverse
CDF are reported only for trajectories satisfying $2\eta<\gamma$.
Accordingly, the relevant conditioning constant is
$c_{\gamma-2\eta}$.

Figure~\ref{fig:qhat-small-m} illustrates the resulting
conditioning--coverage trade-off. Increasing $m$ brings
$c_{\gamma-2\eta}$ closer to $c_\gamma$, reduces the stability factor
$c_{\gamma-2\eta}\eta$, and generally decreases both inversion errors and
the enlargement of the relevance set. Missing segments correspond to
thresholds for which fewer than ten trajectories satisfy
$2\eta<\gamma$. They become more frequent as $K$ increases because a
uniformly accurate estimate of all sampling probabilities becomes harder.
All results use $T=2000$ and are averaged over $100$ independent
trajectories.

\begin{figure*}[t]
    \centering
    \begin{minipage}[t]{0.325\textwidth}
        \centering
        \includegraphics[width=\linewidth]{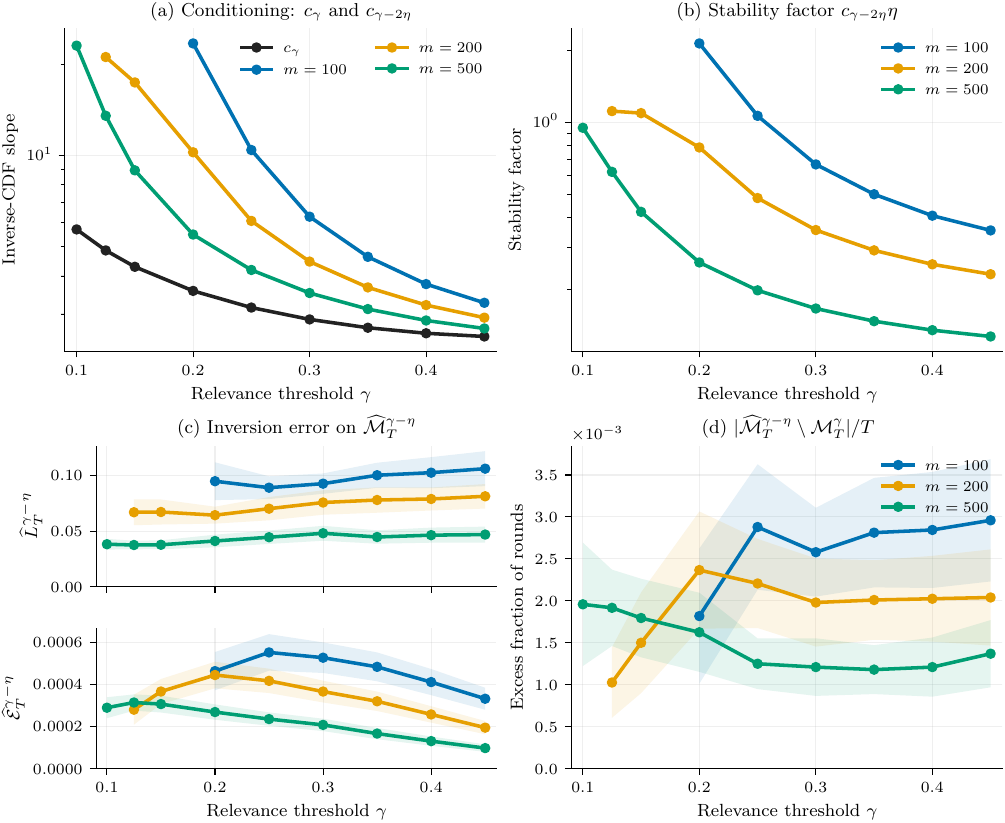}
        \smallskip

        {\small $K=2$}
    \end{minipage}
    \hfill
    \begin{minipage}[t]{0.325\textwidth}
        \centering
        \includegraphics[width=\linewidth]{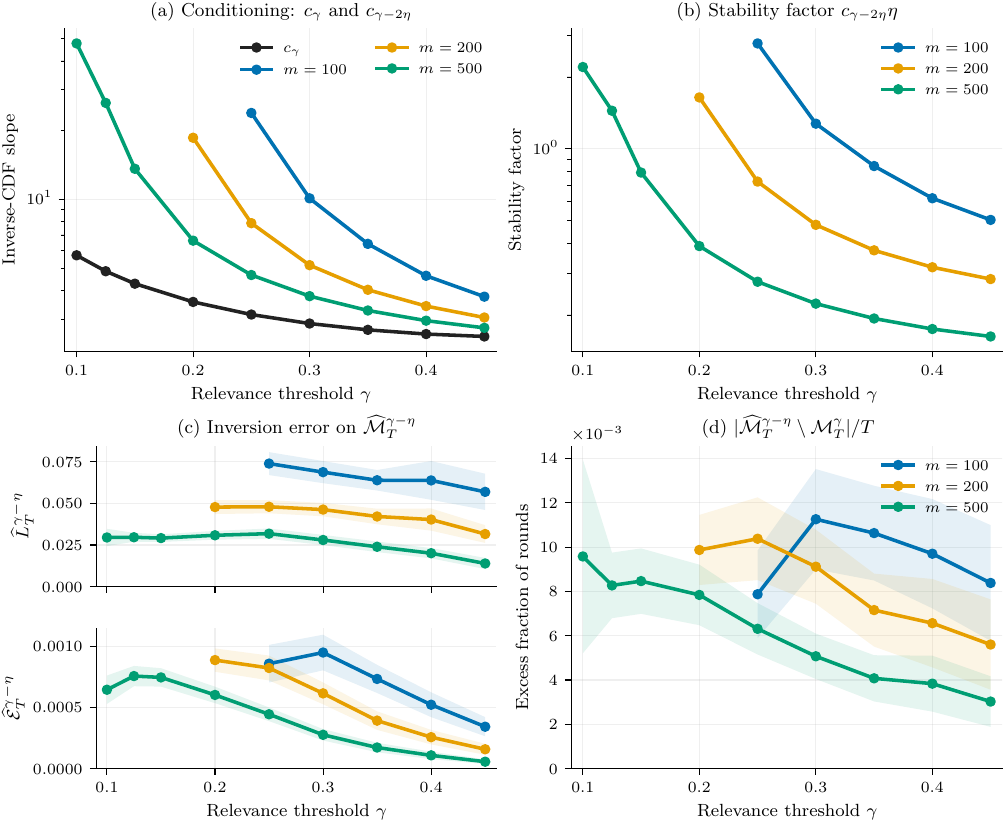}
        \smallskip

        {\small $K=4$}
    \end{minipage}
    \hfill
    \begin{minipage}[t]{0.325\textwidth}
        \centering
        \includegraphics[width=\linewidth]{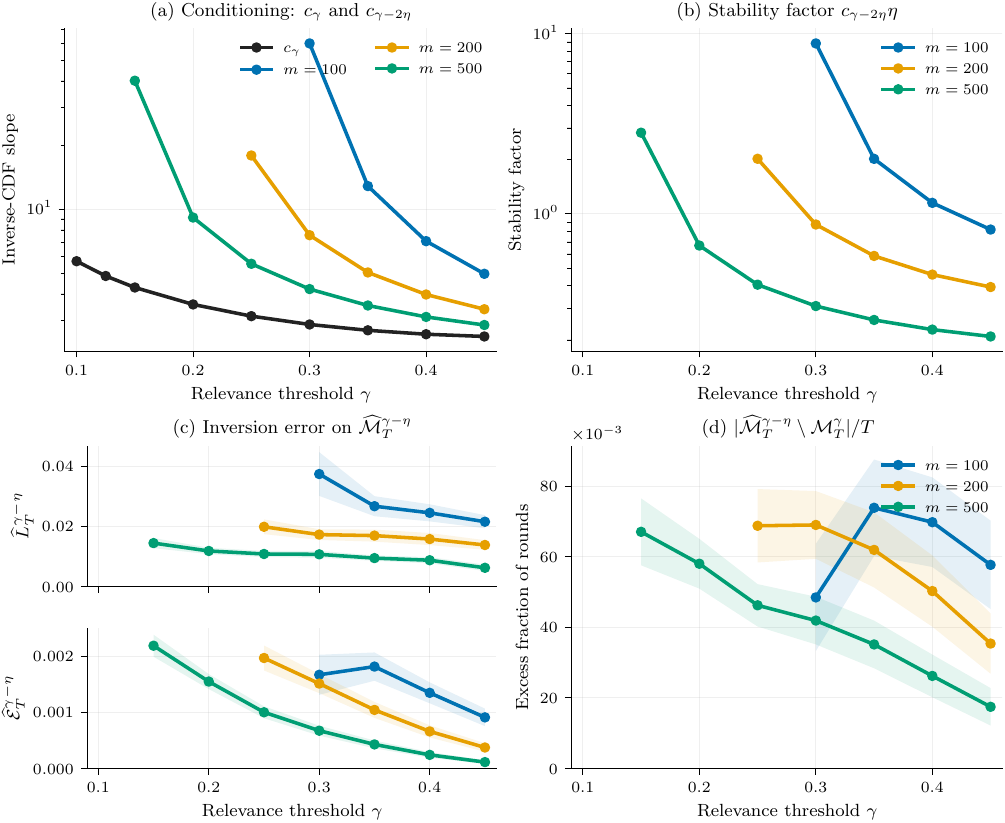}
        \smallskip

        {\small $K=8$}
    \end{minipage}

    \caption{
        Conservative non-batched inversion for $\Phi$-LS with
        $K\in\{2,4,8\}$ (left to right) and
        $m\in\{100,200,500\}$. Panel~(a) compares the intrinsic
        inverse-CDF slope $c_\gamma$ with the median effective slope
        $c_{\gamma-2\eta}$. Panel~(b) reports the median stability factor
        $c_{\gamma-2\eta}\eta$. Panel~(c) reports the practical conditional
        and coverage-weighted inversion errors,
        $\widehat L_T^{\gamma-\eta}$ and
        $\widehat{\mathcal E}_T^{\gamma-\eta}$, computed on
        $\widehat{\mathcal M}_T^{\gamma-\eta}$. Panel~(d) reports the excess
        coverage
        $|\widehat{\mathcal M}_T^{\gamma-\eta}
        \setminus\mathcal M_T^\gamma|/T$.
        Curves are shown only when at least ten of the $100$ trajectories
        satisfy $2\eta<\gamma$; shaded regions in Panels~(c)--(d) are
        pointwise $95\%$ confidence intervals.
    }
    \label{fig:qhat-small-m}
\end{figure*}

\paragraph{Inversion for $v=\Psi$.}
We repeat the same experiment using $\Psi$-LS trajectories on the three
environments above. At every round, the estimates with
$m\in\{100,200,500\}$ are constructed from nested prefixes of the same
sequence of $500$ auxiliary categorical samples. The practical relevance
set and recovered gaps are computed using $\widehat q$, whereas the exact
sampling probabilities are used only to evaluate $\eta$, construct the
oracle set $\mathcal M_T^\gamma$, and verify the exact inversion identity.

Figure~\ref{fig:psi-qhat-small-m} shows the stronger sensitivity of
$\Psi^{-1}$ near the boundary of the probability simplex. Consequently,
$c_{\gamma-2\eta}$ and $c_{\gamma-2\eta}\eta$ are generally larger than
their $\Phi$ counterparts, and probability-estimation errors undergo greater
amplification. Increasing $m$ improves conditioning and reduces both the
practical inversion error and the excess coverage. As for $\Phi$-LS,
missing low-threshold segments indicate that fewer than ten trajectories
satisfy $2\eta<\gamma$; this effect is most pronounced for $K=8$ and
$m=100$. Exact-probability inversion recovers the internal $\Psi$-LS gaps
up to numerical precision.

\begin{figure*}[t]
    \centering
    \begin{minipage}[t]{0.325\textwidth}
        \centering
        \includegraphics[width=\linewidth]{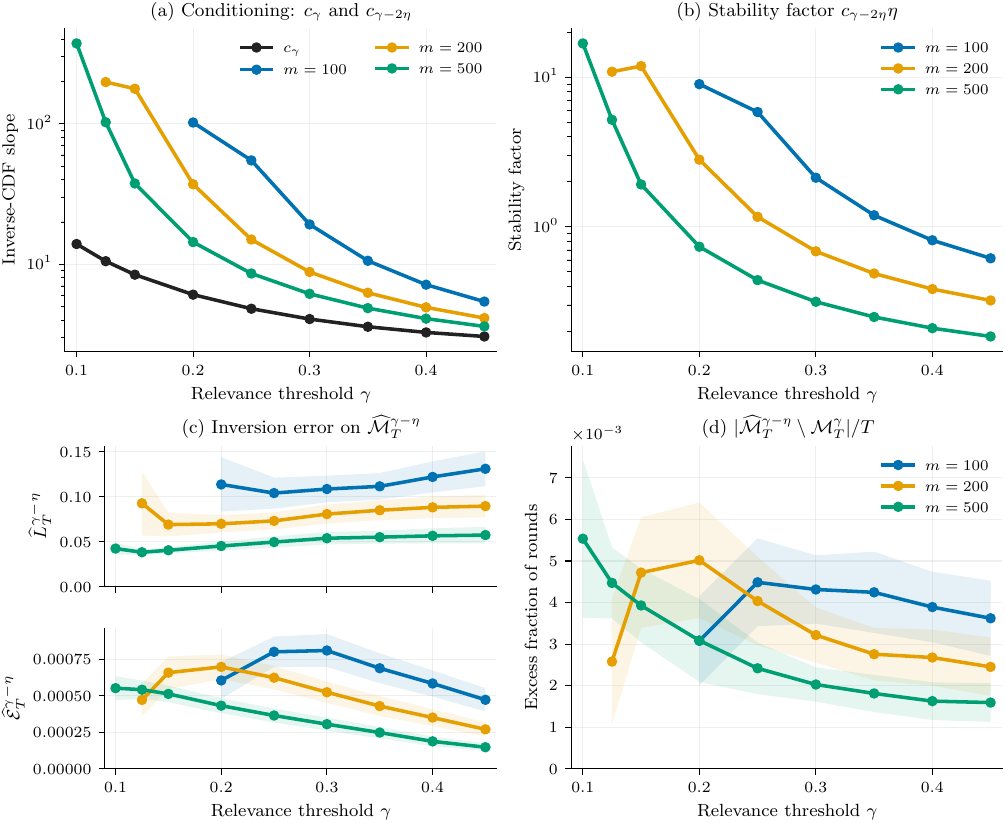}
        \smallskip

        {\small $K=2$}
    \end{minipage}
    \hfill
    \begin{minipage}[t]{0.325\textwidth}
        \centering
        \includegraphics[width=\linewidth]{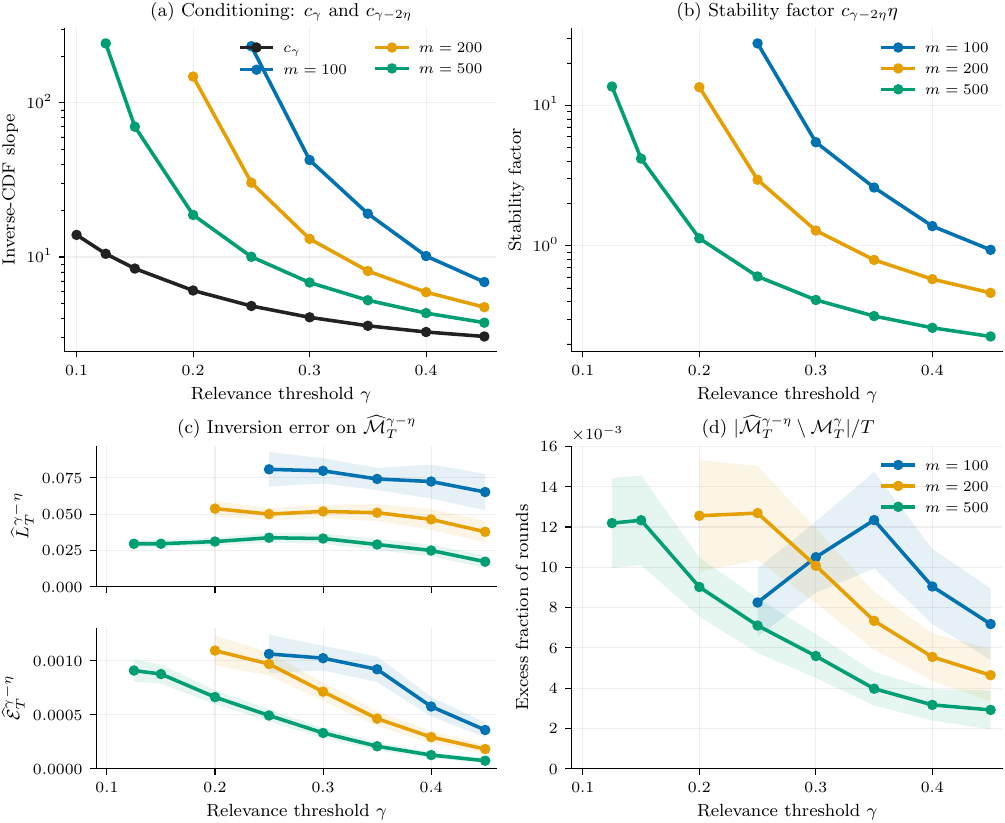}
        \smallskip

        {\small $K=4$}
    \end{minipage}
    \hfill
    \begin{minipage}[t]{0.325\textwidth}
        \centering
        \includegraphics[width=\linewidth]{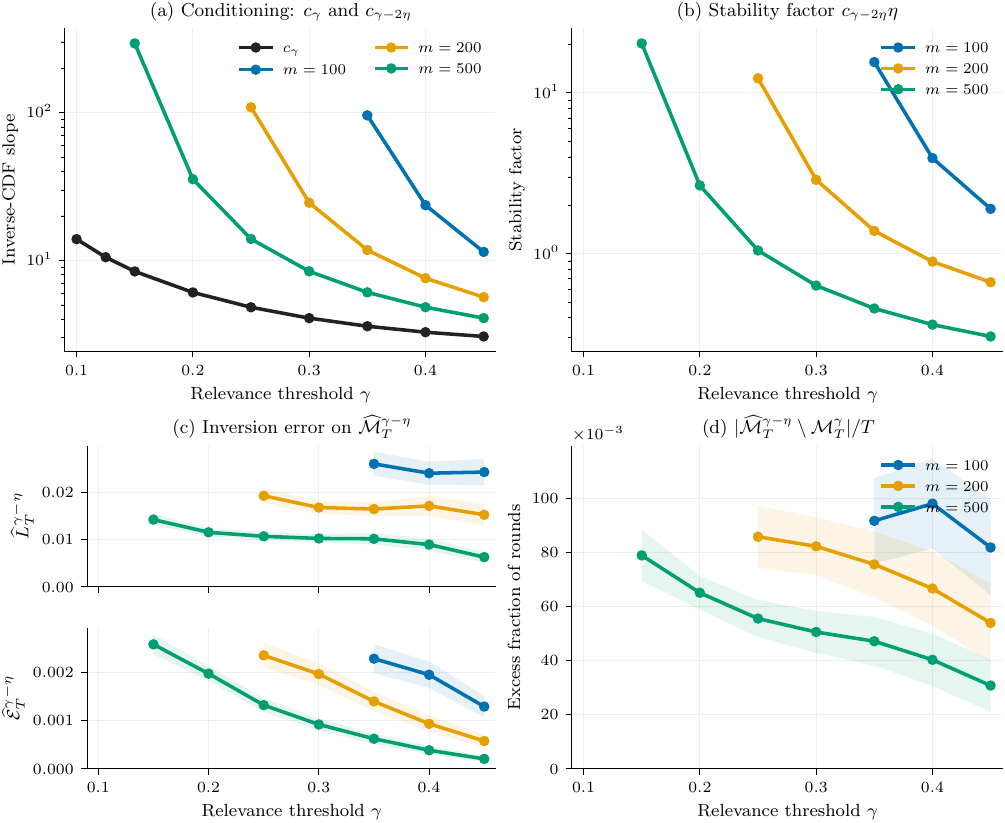}
        \smallskip

        {\small $K=8$}
    \end{minipage}

    \caption{%
        Conservative non-batched inversion for $\Psi$-LS with
        $K\in\{2,4,8\}$ (left to right) and
        $m\in\{100,200,500\}$. Panels~(a)--(d) report, respectively,
        $c_\gamma$ and the median $c_{\gamma-2\eta}$, the median stability
        factor $c_{\gamma-2\eta}\eta$, the practical inversion errors
        $\widehat L_T^{\gamma-\eta}$ and
        $\widehat{\mathcal E}_T^{\gamma-\eta}$ on
        $\widehat{\mathcal M}_T^{\gamma-\eta}$, and the excess coverage
        $|\widehat{\mathcal M}_T^{\gamma-\eta}
        \setminus\mathcal M_T^\gamma|/T$.
        Curves are shown only when at least ten trajectories satisfy
        $2\eta<\gamma$; shaded regions in Panels~(c)--(d) are pointwise
        $95\%$ confidence intervals over the valid trajectories.
    }
    \label{fig:psi-qhat-small-m}
\end{figure*}

\subsection{Batched algorithms}
\paragraph{Batched $\Phi$-LS.}
We also evaluate the role of the threshold when the demonstrator updates its
sampling policy only at batch boundaries. Within a batch of size $n$, the
sampling distribution is fixed, and the inverse learner estimates it directly
from the observed action frequencies:
\[
\hat q_{a,b}
=
\frac{1}{n}
\sum_{t\in\mathcal B_b}
\mathbbm{1}\{A_t=a\}.
\]
We consider $n\in\{500,1000,2000\}$ and use $\widehat q$ both to determine the
practical relevance set and to perform the inversion. The exact frozen policy
$q$ is used only to compute
\[
\eta=\max_{t,a}\left|\widehat q_{a,t}-q_{a,t}\right|
\]
and to provide an oracle coverage comparison.
Figure~\ref{fig:batched-phi} reports the resulting conditioning quantities,
inversion errors, and coverage. Quantities involving $c_{\gamma-2\eta}$ are
displayed only when $2\eta<\gamma$, so missing segments correspond to
configurations outside the stability regime of the theoretical guarantee.
Increasing the batch size improves the frequency estimate, generally reduces
the inversion error, and brings the $\widehat q$-based relevance set closer
to the exact-$q$ set.All results are
averaged over $100$ independent trajectories of length $T=20000$.

\begin{figure*}[t]
    \centering
    \begin{minipage}[t]{0.325\textwidth}
        \centering
        \includegraphics[width=\linewidth]{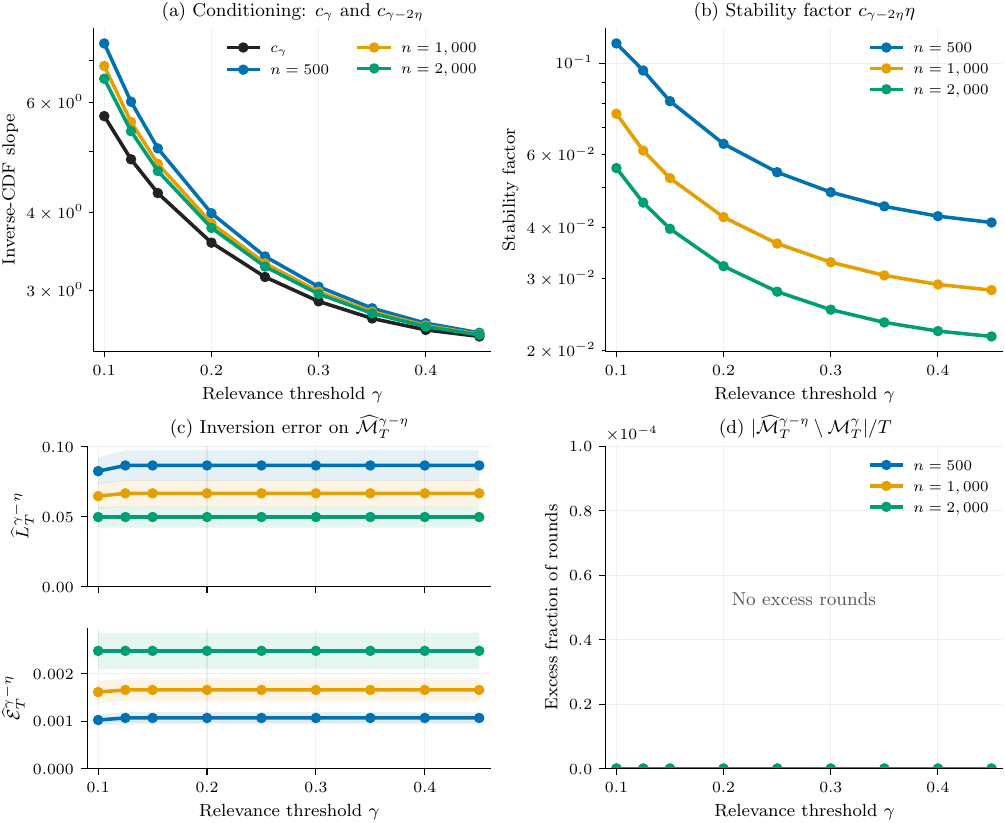}

        \medskip
        {\normalsize\bfseries $K=2$}
    \end{minipage}
    \hfill
    \begin{minipage}[t]{0.325\textwidth}
        \centering
        \includegraphics[width=\linewidth]{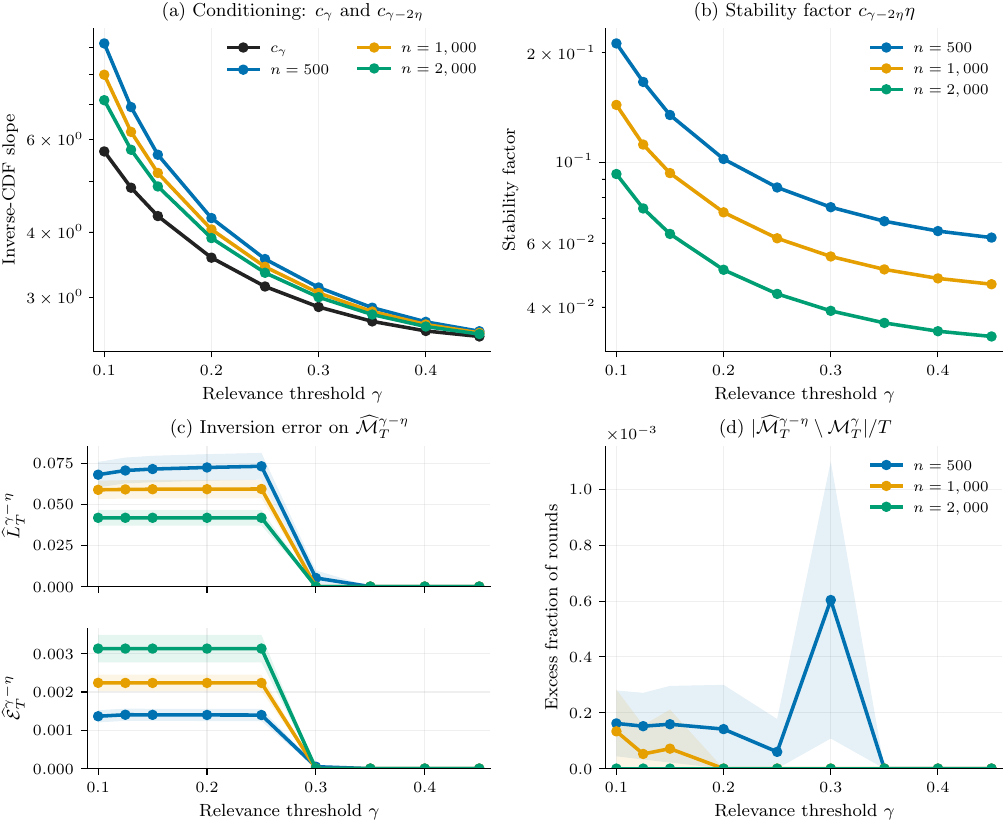}

        \medskip
        {\normalsize\bfseries $K=4$}
    \end{minipage}
    \hfill
    \begin{minipage}[t]{0.325\textwidth}
        \centering
        \includegraphics[width=\linewidth]{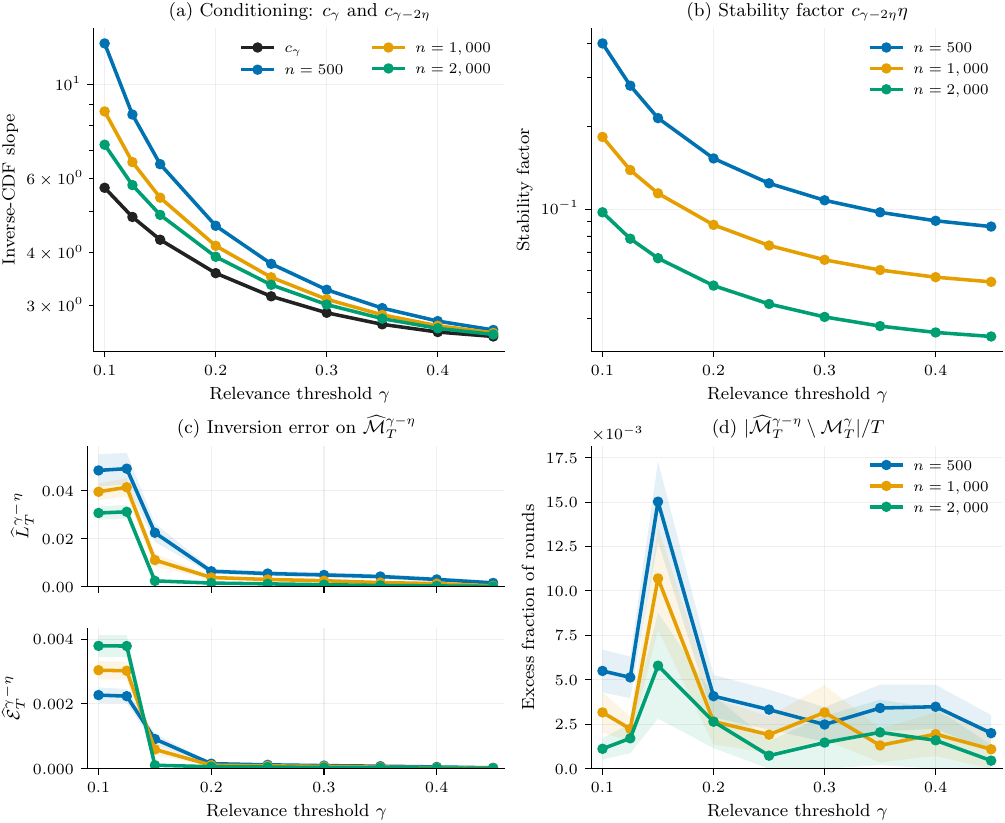}

        \medskip
        {\normalsize\bfseries $K=8$}
    \end{minipage}

    \caption{%
        Batched iLITE-$\Phi$ conditioning and inversion results for
        $K\in\{2,4,8\}$ (left to right), $T=20000$, and batch sizes
        $n\in\{500,1000,2000\}$. Panel~(a) compares $c_\gamma$ with $c_{\gamma-2\eta}$, and
    Panel~(b) reports the stability factor $c_{\gamma-2\eta}\eta$.
    Panel~(c) shows the conditional and coverage-weighted inversion errors,
    $\widehat L_T^{\gamma-\eta}$ and
    $\widehat{\mathcal E}_T^{\gamma-\eta}$, computed on the practical
    relevance set $\widehat{\mathcal M}_T^{\gamma-\eta}$.
    Panel~(d) reports the excess coverage
    $|\widehat{\mathcal M}_T^{\gamma-\eta}
    \setminus\mathcal M_T^\gamma|/T$.
    Quantities are reported only for trajectories satisfying
    $2\eta<\gamma$; shaded regions in Panels~(c)--(d) are pointwise
    $95\%$ confidence intervals.
    }
    \label{fig:batched-phi}
\end{figure*}

\paragraph{Batched $\Psi$-LS.}
Figure~\ref{fig:batched-psi} repeats the conditioning experiment for batched
$\Psi$-LS. The learner updates its posterior only at batch boundaries, while
the observer estimates the sampling probabilities using the empirical action
frequencies within each batch.

\begin{figure*}[t]
    \centering
    \begin{minipage}[t]{0.325\textwidth}
        \centering
        \includegraphics[width=\linewidth]{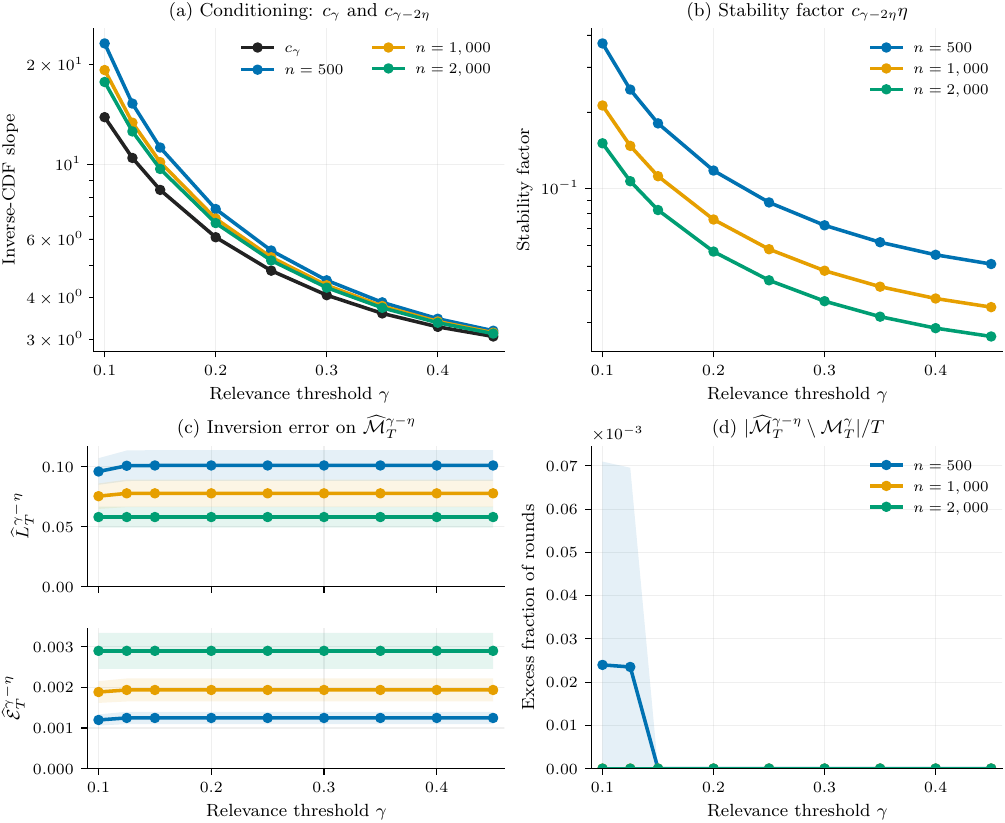}

        \medskip
        {\normalsize\bfseries $K=2$}
    \end{minipage}
    \hfill
    \begin{minipage}[t]{0.325\textwidth}
        \centering
        \includegraphics[width=\linewidth]{Plots/batched_psi_ls_conditioning_4arms_T20000.pdf}

        \medskip
        {\normalsize\bfseries $K=4$}
    \end{minipage}
    \hfill
    \begin{minipage}[t]{0.325\textwidth}
        \centering
        \includegraphics[width=\linewidth]{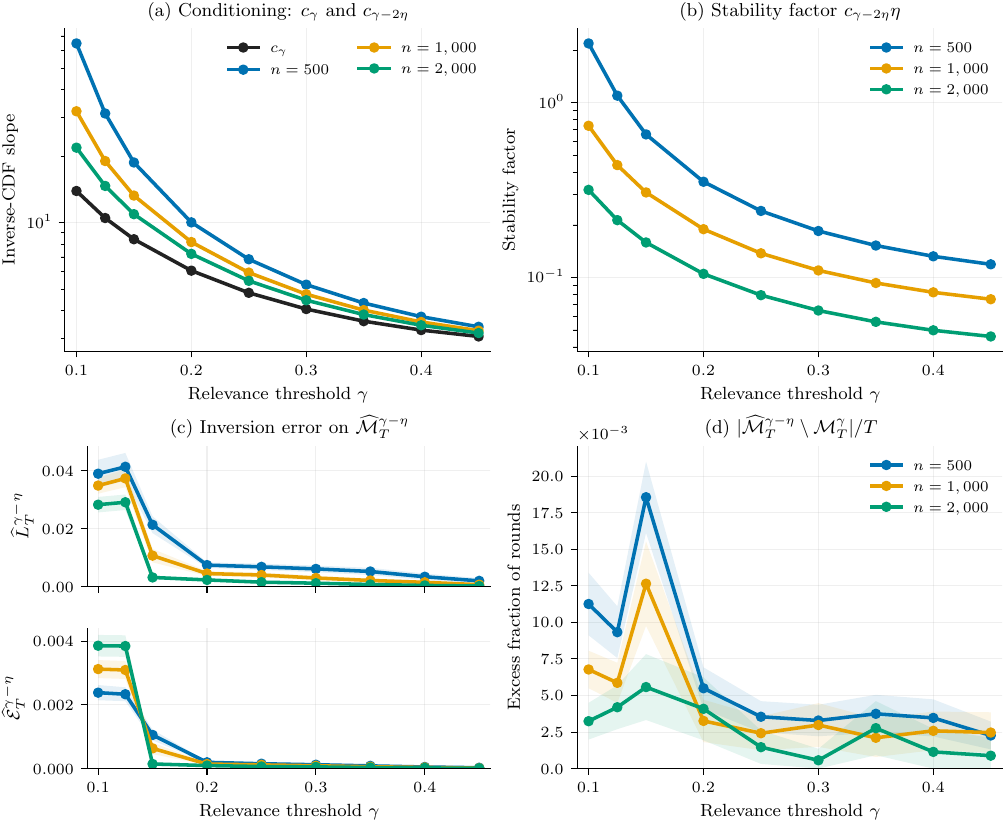}

        \medskip
        {\normalsize\bfseries $K=8$}
    \end{minipage}

    \caption{%
        Batched $\Psi$-LS conditioning experiment for $K\in\{2,4,8\}$
        (left to right), $T=20000$, and batch sizes
        $n\in\{500,1000,2000\}$. Panel~(a) compares $c_\gamma$ with $c_{\gamma-2\eta}$, and
    Panel~(b) reports the stability factor $c_{\gamma-2\eta}\eta$.
    Panel~(c) shows the conditional and coverage-weighted inversion errors,
    $\widehat L_T^{\gamma-\eta}$ and
    $\widehat{\mathcal E}_T^{\gamma-\eta}$, computed on the practical
    relevance set $\widehat{\mathcal M}_T^{\gamma-\eta}$.
    Panel~(d) reports the excess coverage
    $|\widehat{\mathcal M}_T^{\gamma-\eta}
    \setminus\mathcal M_T^\gamma|/T$.
    Quantities are reported only for trajectories satisfying
    $2\eta<\gamma$; shaded regions in Panels~(c)--(d) are pointwise
    $95\%$ confidence intervals.
    }
    \label{fig:batched-psi}
\end{figure*}

\subsection{Sliding windows}
\label[appendix]{app:sliding_windows}
We also study inversion when the observer has access only to the action
trajectory and estimates the current sampling probabilities from a causal
sliding window. For window length $n\in\{20,50,100\}$, we use
\[
\widehat q_{a,t}^{(n)}
=
\frac{1}{n}\sum_{s=t-n}^{t-1}\mathbbm{1}\{A_s=a\},
\qquad t>n.
\]
Thus, $\widehat q_{a,t}^{(n)}$ uses the frequencies of arm $a$ in the
previous $n$ actions and does not use the current action. This provides a
local, trajectory-based way to reflect temporal dependence and changes in
the policy over time. Since the policy is not frozen within a sliding
window, the estimate can reflect both finite-window variation and policy
drift. Because actions are generated adaptively and the sampling policy may change within each window, we do not claim a theoretical guarantee for the accuracy of this estimator. We evaluate its empirical accuracy in simulation by comparing \(\widehat q_t\) with the known \(q_t\); deriving guarantees for this setting is left to future work.

We evaluate all window sizes on the common interval $t\geq101$, and report
results for $T=20{,}000$ over 100 independent trajectories. At each evaluated
round, we compute
$\eta_t=\max_{a\in[K]}|\widehat q_{a,t}-q_{a,t}|$ and check the local
condition $2\eta_t<\gamma$, using $\gamma=0.20$. The action-frequency
estimate $\widehat q$ is used for practical relevance and gap inversion;
the exact probabilities $q$ are used only to evaluate $\eta_t$ and to
compute the oracle coverage curve. In each figure, panel (a) shows the
pointwise probability error and the threshold $\gamma/2$. Panel (b) compares
the fraction of trajectories satisfying the local condition at round $t$
with the stricter running condition
$2\max_{101\leq s\leq t}\eta_s<\gamma$. Panel (c) reports the running
conditional inversion error and its coverage-weighted counterpart on
locally valid, empirically relevant rounds. Panel (d) compares relevance
coverage computed from $\widehat q$ with oracle coverage computed from $q$,
both restricted to locally valid rounds. The pointwise error generally
shrinks as the policies settle, while the running condition remains more
restrictive because an early large error continues to affect the running
maximum. Coverage-weighted error falls as relevant rounds become less
frequent. Figures~\ref{fig:sliding-phi-k2}--\ref{fig:sliding-phi-k8}
show the results for $\Phi$-LS, and
Figures~\ref{fig:sliding-psi-k2}--\ref{fig:sliding-psi-k8} show those for
$\Psi$-LS.

\begin{figure*}[t]
    \centering
    \includegraphics[width=\textwidth]
    {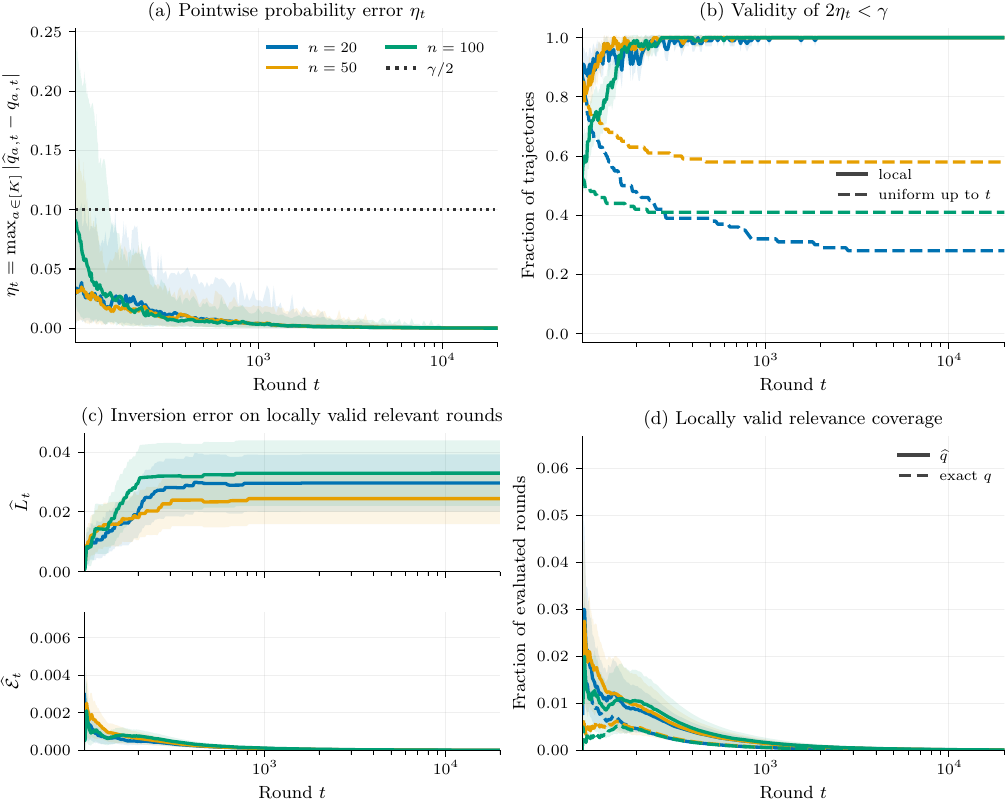}
    \caption{Sliding-window inversion for $\Phi$-LS with $K=2$,
    $T=20{,}000$, $\gamma=0.20$, and $n\in\{20,50,100\}$, averaged over
    100 trajectories. Panels (a)--(d) show pointwise probability error,
    local and running-uniform validity, inversion error, and empirical
    versus oracle relevance coverage, respectively.}
    \label{fig:sliding-phi-k2}
\end{figure*}

\begin{figure*}[t]
    \centering
    \includegraphics[width=\textwidth]
    {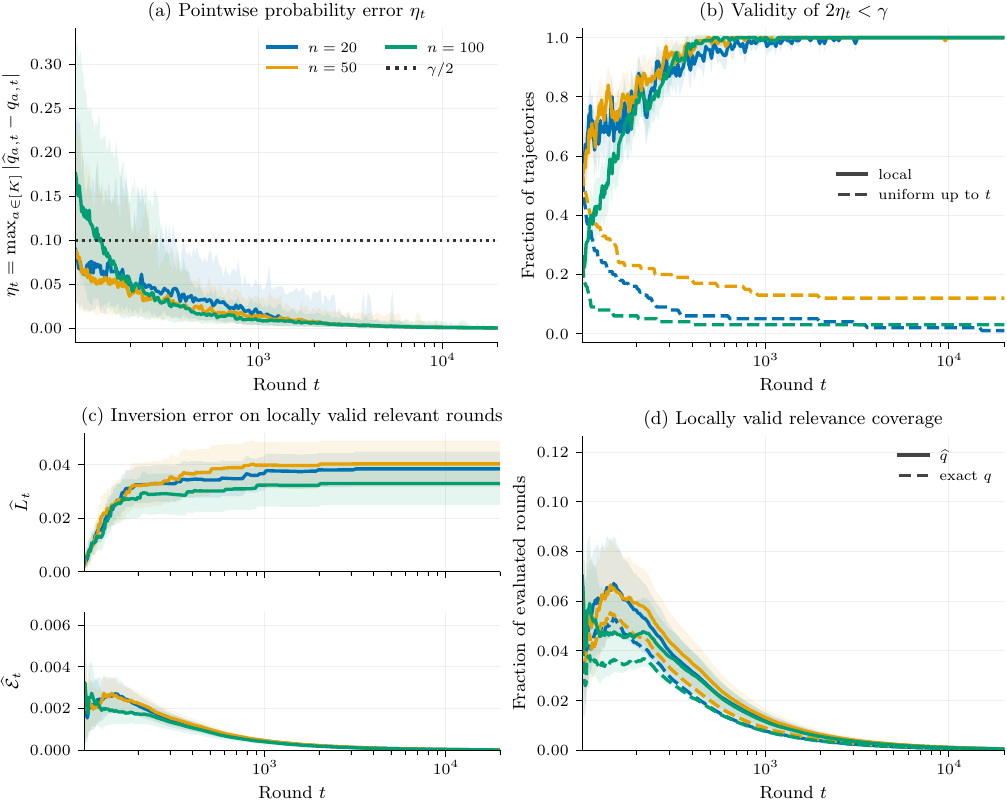}
    \caption{Sliding-window inversion for $\Phi$-LS with $K=4$.
    The remaining settings and panels are the same as in
    Figure~\ref{fig:sliding-phi-k2}.}
    \label{fig:sliding-phi-k4}
\end{figure*}

\begin{figure*}[t]
    \centering
    \includegraphics[width=\textwidth]
    {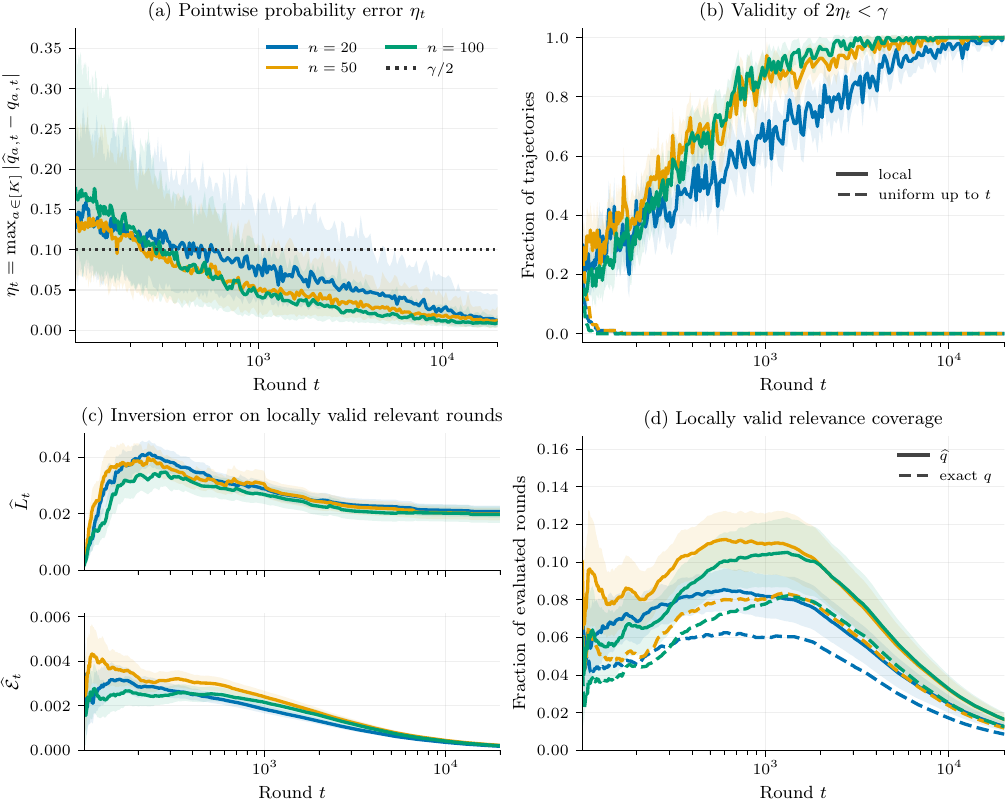}
    \caption{Sliding-window inversion for $\Phi$-LS with $K=8$.
    The remaining settings and panels are the same as in
    Figure~\ref{fig:sliding-phi-k2}.}
    \label{fig:sliding-phi-k8}
\end{figure*}

\begin{figure*}[t]
    \centering
    \includegraphics[width=\textwidth]
    {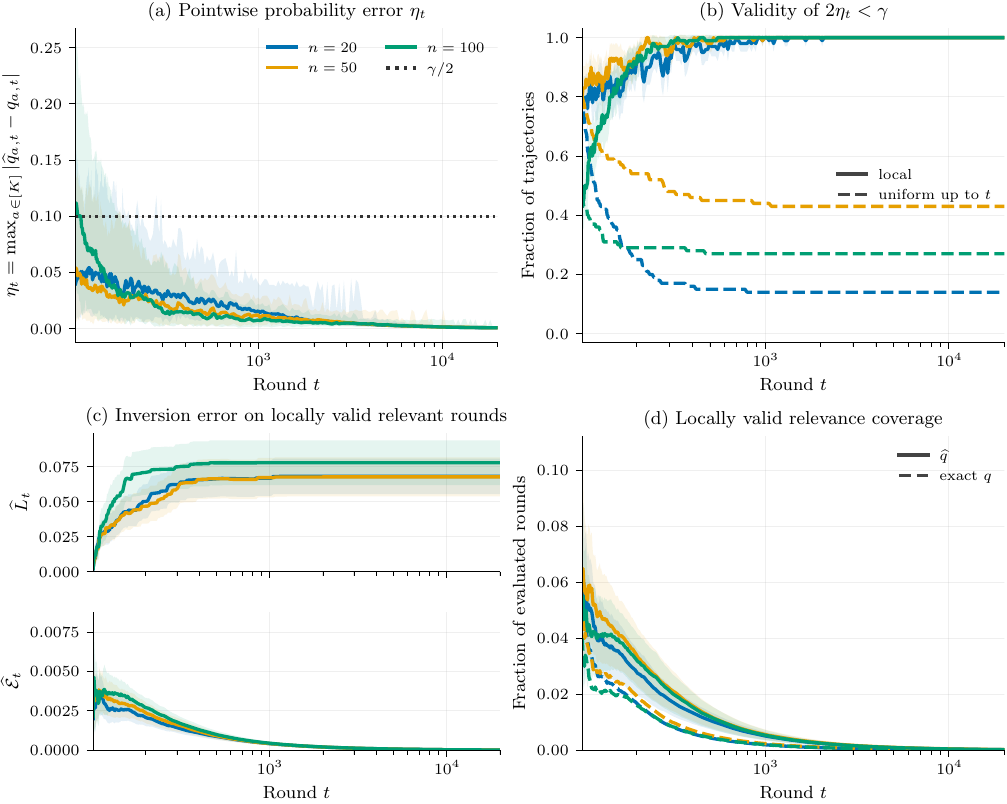}
    \caption{Sliding-window inversion for $\Psi$-LS with $K=2$,
    $T=20{,}000$, $\gamma=0.20$, and $n\in\{20,50,100\}$, averaged over
    100 trajectories. Panels (a)--(d) show pointwise probability error,
    local and running-uniform validity, inversion error, and empirical
    versus oracle relevance coverage, respectively.}
    \label{fig:sliding-psi-k2}
\end{figure*}

\begin{figure*}[t]
    \centering
    \includegraphics[width=\textwidth]
    {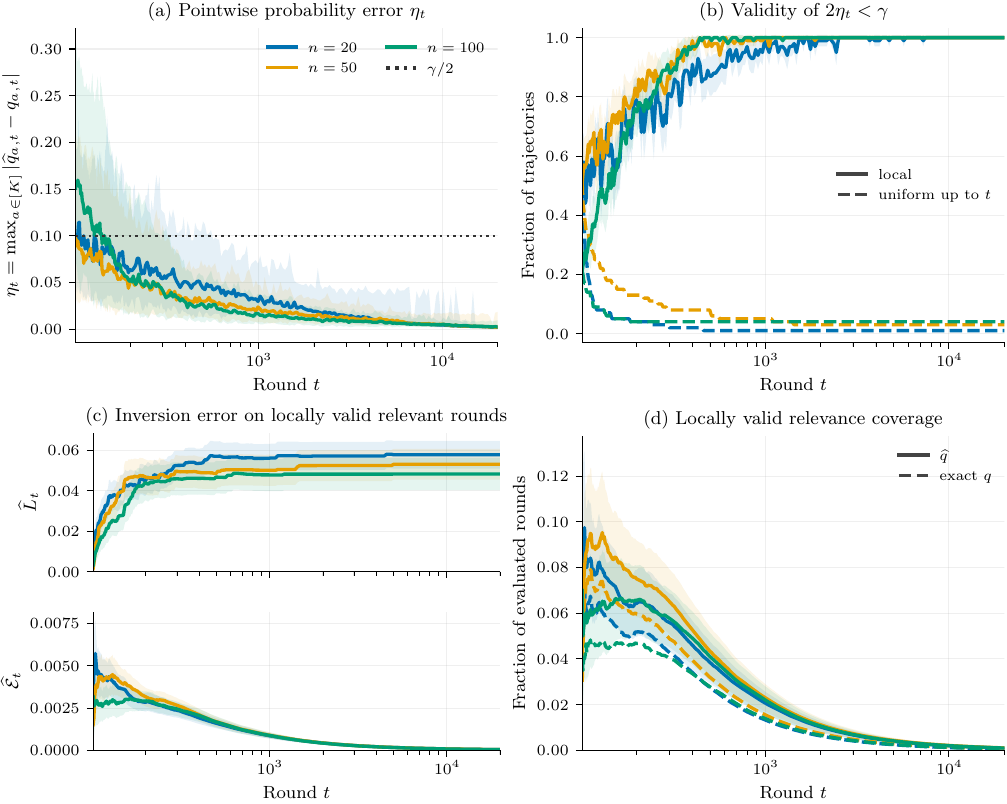}
    \caption{Sliding-window inversion for $\Psi$-LS with $K=4$.
    The remaining settings and panels are the same as in
    Figure~\ref{fig:sliding-psi-k2}.}
    \label{fig:sliding-psi-k4}
\end{figure*}

\begin{figure*}[t]
    \centering
    \includegraphics[width=\textwidth]
    {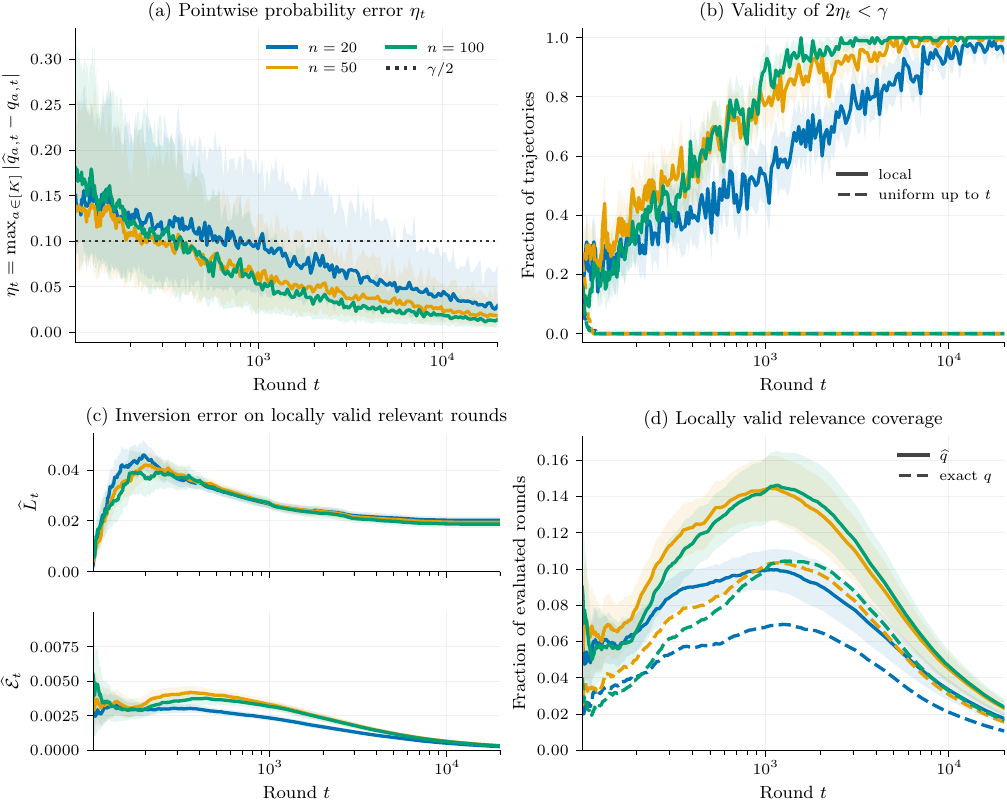}
    \caption{Sliding-window inversion for $\Psi$-LS with $K=8$.
    The remaining settings and panels are the same as in
    Figure~\ref{fig:sliding-psi-k2}.}
    \label{fig:sliding-psi-k8}
\end{figure*}

\subsection{$(\varepsilon,\tilde\varepsilon)$-LS family}
\label[appendix]{sec:epsilon-ls-experiment}

We empirically test~\cref{def:demos_family} for some common algorithms: TS, Maillard sampling (MS)~\citep{maillard2021fundamental,bian2022maillard}, and Boltzmann-Gumbel exploration (BGE)~\citep{cesa2017boltzmann} with the $\Phi$-LS and $\Psi$-LS references.
For each pair, we evaluate $|\Delta_{a_t,t}^{LS}-\Delta_{a_t,t}^\mathcal{A}|$ and
$\max_{a}|p_{a,t}-q_{a,t}|$ on matched histories.  We also compare
$\Phi$-LS and $\Psi$-LS directly. To examine how these quantities change with the number of arms, we consider a $K=2$ Gaussian bandit with mean vector
$\mu=(1/2,1)$, a $K=4$ bandit with mean vector
$\mu=(0,1/3,2/3,1)$, and a $K=8$ Gaussian bandit with
\[
\mu_a=1-0.1(a-1), \qquad a\in[K],
\]
corresponding to $\mu=(1,0.9,\ldots,0.3)$. In all settings, the observation variance is $\alpha^2=1$.

\begin{figure*}[h]
  \centering
  \begin{subfigure}[t]{0.49\textwidth}
    \includegraphics[width=\linewidth]{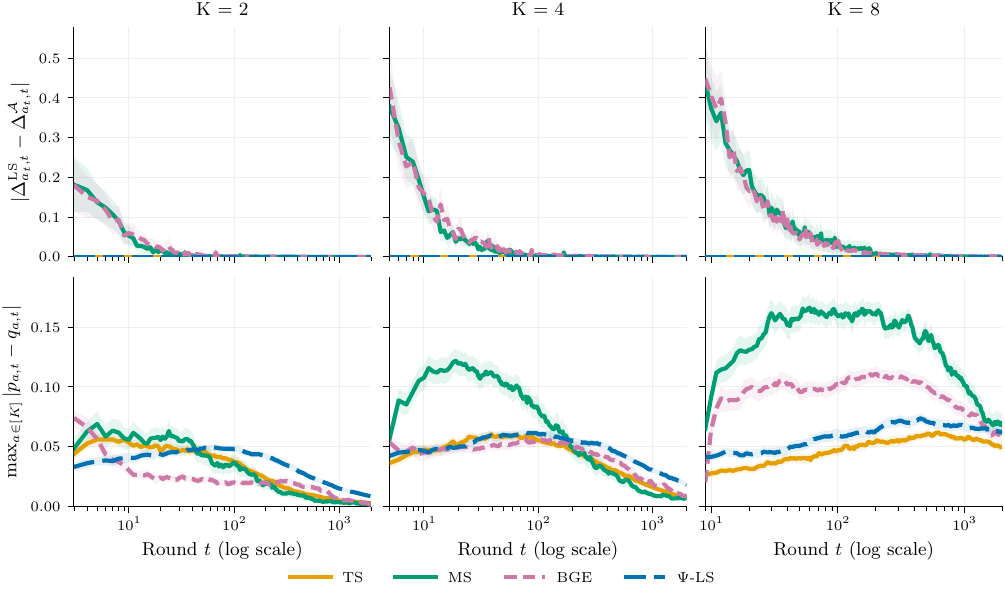}
    \caption{$\Phi$-LS reference.}
    \label{fig:epsilon-ls-phi}
  \end{subfigure}\hfill
  \begin{subfigure}[t]{0.49\textwidth}
    \includegraphics[width=\linewidth]{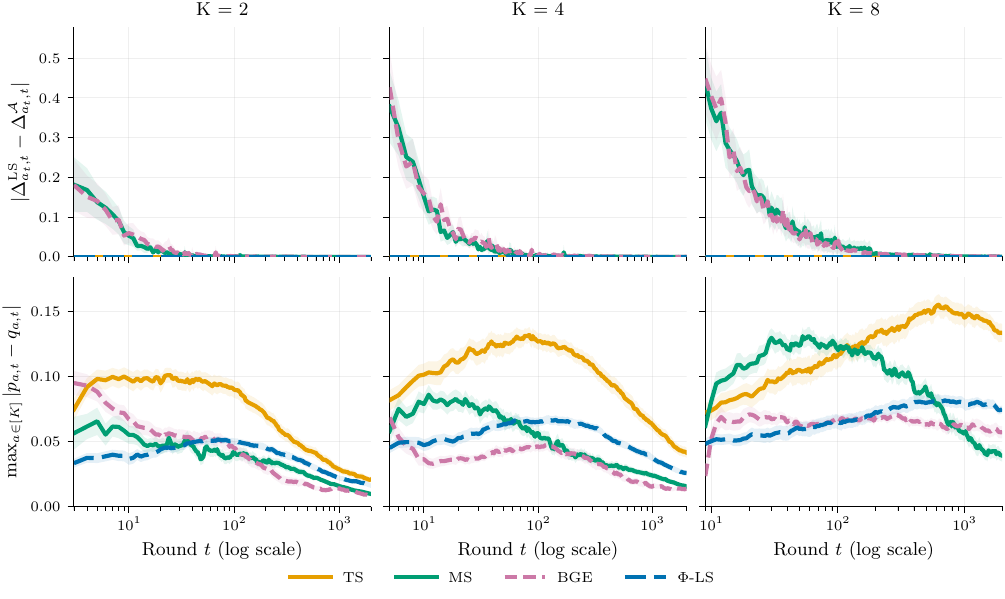}
    \caption{$\Psi$-LS reference.}
    \label{fig:epsilon-ls-psi}
  \end{subfigure}
  \caption{Same-history discrepancies relative to the $\Phi$-LS reference (left) and the $\Psi$-LS reference (right), for $K\in\{2,4,8\}$ and $T=2000$. The upper panels report the played-arm gap discrepancy
$|\Delta^{\mathrm{ref}}_{a_t,t}-\Delta^{\mathcal A}_{a_t,t}|$, while the lower panels report the maximum sampling-policy discrepancy
$\max_{a\in[K]}|p^{\mathcal A}_{a,t}-q^{\mathrm{ref}}_{a,t}|$.
For each candidate trajectory, the reference policy is evaluated on the same realized action-reward history using exact sampling probabilities. Curves show the mean over 100 simulations and shaded regions show pointwise 95\% confidence intervals. The horizontal axis begins at $t=K$, marking the end of the forced initialization used by MS and BGE.}
  \label{fig:epsilon-ls}
\end{figure*}

Figure~\ref{fig:epsilon-ls} separates discrepancies in the internal gap estimates from discrepancies in the resulting sampling policies. Since TS, $\Phi$-LS, and $\Psi$-LS use the same Gaussian posterior update, their same-history gap discrepancies are zero up to numerical precision. In contrast, MS and BGE construct their estimates differently and exhibit a transient gap discrepancy near initialization due to their forced initialization, which becomes larger as $K$ increases but decreases as more observations are collected. The lower panels show that agreement in gap construction does not necessarily imply agreement between sampling policies: TS and the two LS variants remain relatively close, while MS and BGE display larger policy discrepancies, particularly for larger $K$. The horizontal axis begins at $t=K$, marking the end of the forced initialization used by MS and BGE.  Relative to the $\Psi$-LS reference, the discrepancy of TS also increases with $K$. These curves describe how the empirical discrepancies evolve over time; the corresponding trajectory-wise maxima are needed to assess the uniform conditions formally.

\paragraph{End-to-end inversion $v=\Psi$}
We next evaluate the complete inversion procedure for TS, MS, and BGE at the
fixed relevance threshold $\gamma=0.4$.
Figures~\ref{fig:fixed-gamma-psi-k2}--\ref{fig:fixed-gamma-psi-k8} should be interpreted together
with the $(\varepsilon,\widetilde{\varepsilon})$-LS comparison above.
In that comparison, $\varepsilon$ measures the discrepancy between the
demonstrator's sampling distribution and that of its $\Psi$-LS reference,
whereas $\widetilde{\varepsilon}$ measures the discrepancy between their
internal gaps when evaluated on the same realized history. The end-to-end
experiment shows how these two forms of misspecification affect the gaps
recovered by iLITE-$\Psi$. 

Figures~\ref{fig:fixed-gamma-psi-k2}--\ref{fig:fixed-gamma-psi-k8}
show a clear dependence on the number of arms. For $K=2$, the relevance
set is comparatively large and MS exhibits the largest conditional and
coverage-weighted errors, followed by BGE and TS. The same qualitative
ordering persists for $K=4$, although the relevant-round fraction and
$\mathcal E_t^\gamma$ are smaller. For $K=8$, non-reference pulls remain
frequent throughout the horizon, but very few rounds satisfy the relevance
condition at $\gamma=0.4$, since the sampling probability is distributed
among more arms. Consequently, the coverage-weighted error becomes very
small and the conditional-error curves are more irregular.
\begin{figure*}[h]
    \centering
    \includegraphics[width=\textwidth]{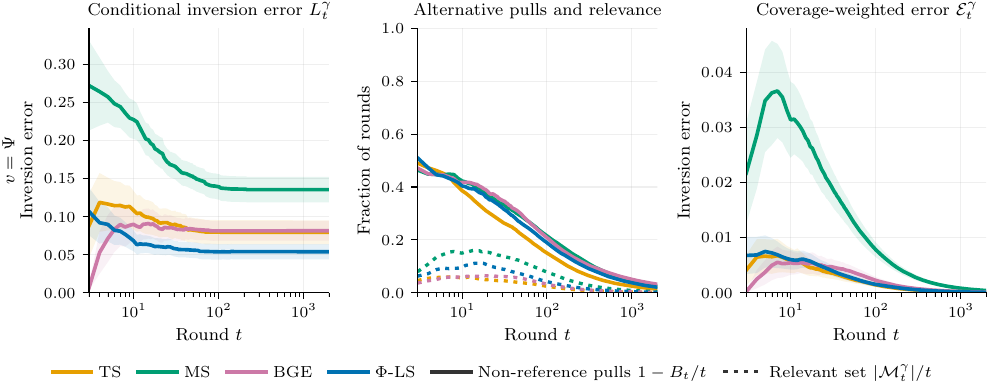}
    \caption{%
        Fixed-threshold inversion with $v=\Psi$ for $K=2$, $T=2000$, and
        $\gamma=0.4$. We invert TS, MS, BGE, and a $\Phi$-LS demonstrator
        using the exact sampling probabilities. The panels report the
        conditional inversion error $L_t^\gamma$, the non-reference pull
        fraction $1-B_t/t$ and relevant-round fraction
        $|\mathcal M_t^\gamma|/t$, and the coverage-weighted error
        $\mathcal E_t^\gamma$. Curves are restricted to trajectories
        satisfying $2\varepsilon<\gamma$; shaded regions are pointwise
        $95\%$ confidence intervals.
    }
    \label{fig:fixed-gamma-psi-k2}
\end{figure*}

\begin{figure*}[h]
    \centering
    \includegraphics[width=\textwidth]{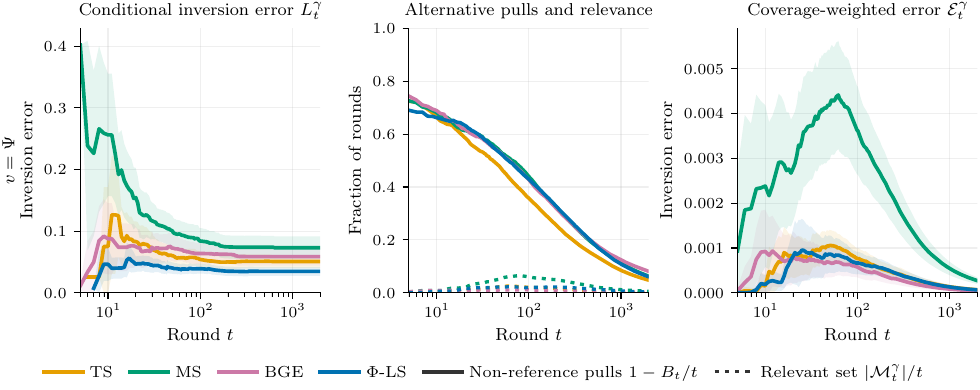}
    \caption{%
        Fixed-threshold inversion with $v=\Psi$ for $K=4$, $T=2000$, and
        $\gamma=0.4$. The panels respectively report $L_t^\gamma$, the
        fractions of non-reference and relevant rounds, and
        $\mathcal E_t^\gamma$. The blue curve corresponds to a $\Phi$-LS
        demonstrator inverted using $\Psi^{-1}$ and therefore isolates
        inverse-CDF misspecification. Exact sampling probabilities are used,
        and only trajectories satisfying $2\varepsilon<\gamma$ contribute.
        Shaded regions are pointwise $95\%$ confidence intervals.
    }
    \label{fig:fixed-gamma-psi-k4}
\end{figure*}

\begin{figure*}[h]
    \centering
    \includegraphics[width=\textwidth]{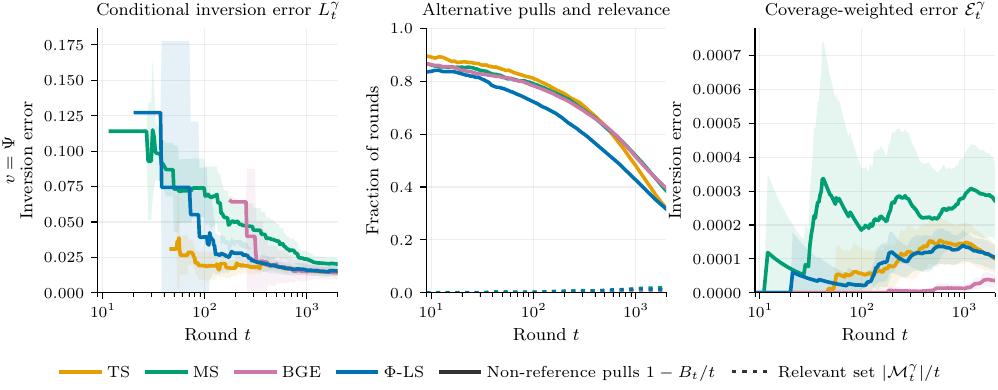}
    \caption{%
        Fixed-threshold inversion with $v=\Psi$ for $K=8$, $T=2000$, and
        $\gamma=0.4$. The three panels show the conditional inversion error,
        the decomposition into non-reference and relevant rounds, and the
        coverage-weighted inversion error. The $\Phi$-LS curve illustrates
        inversion using a misspecified inverse CDF. Results use exact
        sampling probabilities and are restricted to trajectories satisfying
        $2\varepsilon<\gamma$. Shaded regions are pointwise $95\%$
        confidence intervals.
    }
    \label{fig:fixed-gamma-psi-k8}
\end{figure*}

\paragraph{End-to-end inversion for $v=\Phi$.}
Figures~\ref{fig:fixed-gamma-phi-k2}--\ref{fig:fixed-gamma-phi-k8}
complement the corresponding $\Psi$-inversion results. For $K=2$, MS
again exhibits the largest conditional and coverage-weighted errors, while
TS has the smallest error and BGE and $\Psi$-LS lie between them. The same
ordering is broadly preserved for $K=4$, although the relevance set and
$\mathcal E_t^\gamma$ are smaller. Compared with inversion using
$\Psi^{-1}$, the $\Phi$ inverse produces visibly smaller errors for TS and
BGE in these two environments, showing that the choice of inverse CDF
affects how policy misspecification is amplified. Conversely, applying
$\Phi^{-1}$ to a $\Psi$-LS demonstrator produces a nonzero error, mirroring
the error obtained when $\Phi$-LS is inverted using $\Psi^{-1}$.

For $K=8$, both experiments exhibit substantial non-reference play but
limited relevance at $\gamma=0.4$. Hence, the small coverage-weighted errors
primarily reflect the contraction of the relevance set and should not be
interpreted as improved inversion. This qualification is especially
important for MS under $v=\Phi$, for which only three trajectories satisfy
$2\varepsilon<\gamma$. More generally, the $\Phi$ and $\Psi$ results agree
qualitatively on the role of $K$. 

\begin{figure*}[h]
    \centering
    \includegraphics[width=\textwidth]{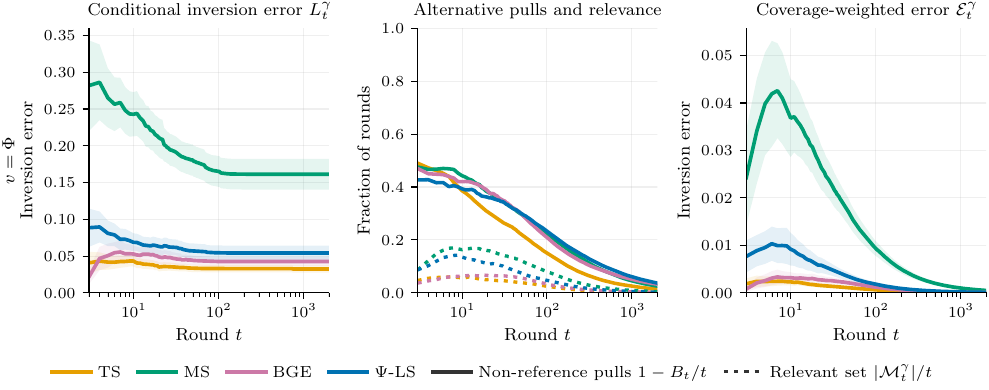}
    \caption{%
        Fixed-threshold inversion with $v=\Phi$ for $K=2$, $T=2000$, and
        $\gamma=0.4$. We invert TS, MS, BGE, and a $\Psi$-LS demonstrator
        using the exact sampling probabilities. The left panel reports the
        conditional inversion error $L_t^\gamma$. The central panel shows
        the fraction of non-reference pulls, $1-B_t/t$, together with the
        relevant-round fraction $|\mathcal M_t^\gamma|/t$. The right panel
        reports the coverage-weighted error $\mathcal E_t^\gamma$.
        Curves are computed over trajectories satisfying
        $2\varepsilon<\gamma$; shaded regions are pointwise $95\%$
        confidence intervals.
    }
    \label{fig:fixed-gamma-phi-k2}
\end{figure*}

\begin{figure*}[h]
    \centering
    \includegraphics[width=\textwidth]{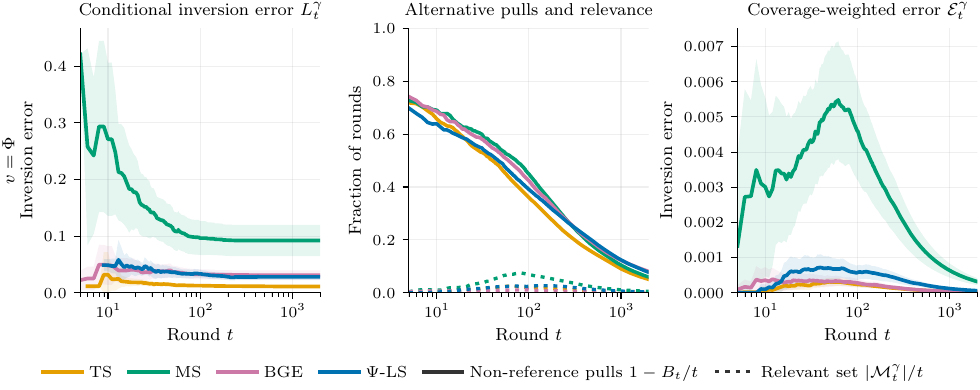}
    \caption{%
        Fixed-threshold inversion with $v=\Phi$ for $K=4$, $T=2000$, and
        $\gamma=0.4$. The panels respectively report the conditional
        inversion error $L_t^\gamma$, the non-reference and relevant-round
        fractions, and the coverage-weighted error $\mathcal E_t^\gamma$.
        The blue curve corresponds to a $\Psi$-LS demonstrator inverted
        using $\Phi^{-1}$. Exact sampling probabilities are used, and only
        trajectories satisfying $2\varepsilon<\gamma$ contribute to each
        curve. Shaded regions are pointwise $95\%$ confidence intervals.
    }
    \label{fig:fixed-gamma-phi-k4}
\end{figure*}

\begin{figure*}[h]
    \centering
    \includegraphics[width=\textwidth]{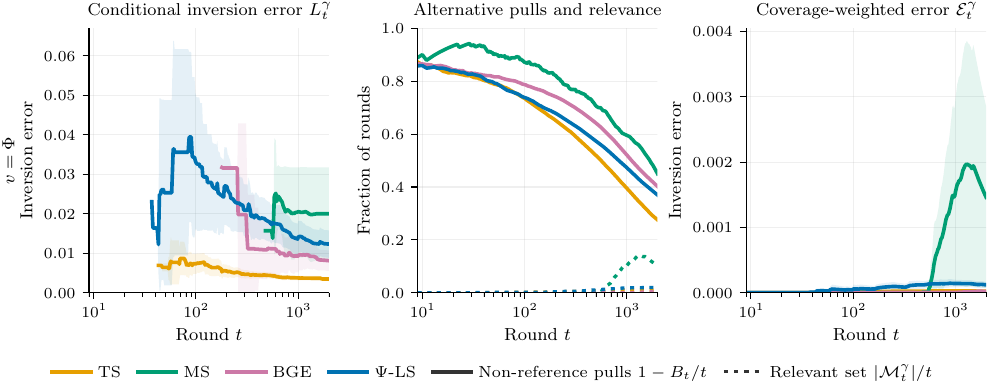}
    \caption{%
        Fixed-threshold inversion with $v=\Phi$ for $K=8$, $T=2000$, and
        $\gamma=0.4$. The three panels show $L_t^\gamma$, the fractions of
        non-reference and relevant rounds, and $\mathcal E_t^\gamma$.
        The $\Psi$-LS curve illustrates inversion under inverse-CDF
        misspecification. Results use exact sampling probabilities and are
        restricted to trajectories satisfying $2\varepsilon<\gamma$.
        For MS, this condition holds for only a small number of trajectories. Shaded regions are
        pointwise $95\%$ confidence intervals.
    }
    \label{fig:fixed-gamma-phi-k8}
\end{figure*}

\subsection{Robustness to misspecification}
We evaluate the effect of misspecifying the inverse learner's Gaussian prior
variance for $K=4$ arms over a horizon of $T=2000$. In each panel in~\cref{fig:prior-variance-phi,fig:prior-variance-psi}, we
multiply the prior variance of one arm by
$\rho\in\{0.25,0.5,2,4\}$, leaving the other arms correctly specified. The
demonstrator retains the correct prior variance
$\sigma_{a,0}^2=1$, while the inverse learner uses
$\widetilde{\sigma}_{a,0}^2=\rho\sigma_{a,0}^2$ for the perturbed arm.
We use the demonstrator's exact sampling probabilities in every inversion,
so the observed error isolates prior misspecification rather than
probability-estimation error. The dashed black curve is the matched-prior
baseline and is zero up to numerical precision. Prior misspecification produces the largest errors during early learning, and these generally decrease as observations accumulate. This is consistent with~\cref{eq:var}: as an arm’s pull count grows, the data-dependent term increasingly dominates the prior-variance contribution, reducing the influence of the initial misspecification. The effect is most pronounced when the best arm’s prior variance is perturbed, because the best arm enters the gap comparison at every relevant round. Results are averaged over 100
independent trajectories; shaded regions show 95\% confidence intervals.

\begin{figure*}[t]
    \centering
    \includegraphics[width=\textwidth]
    {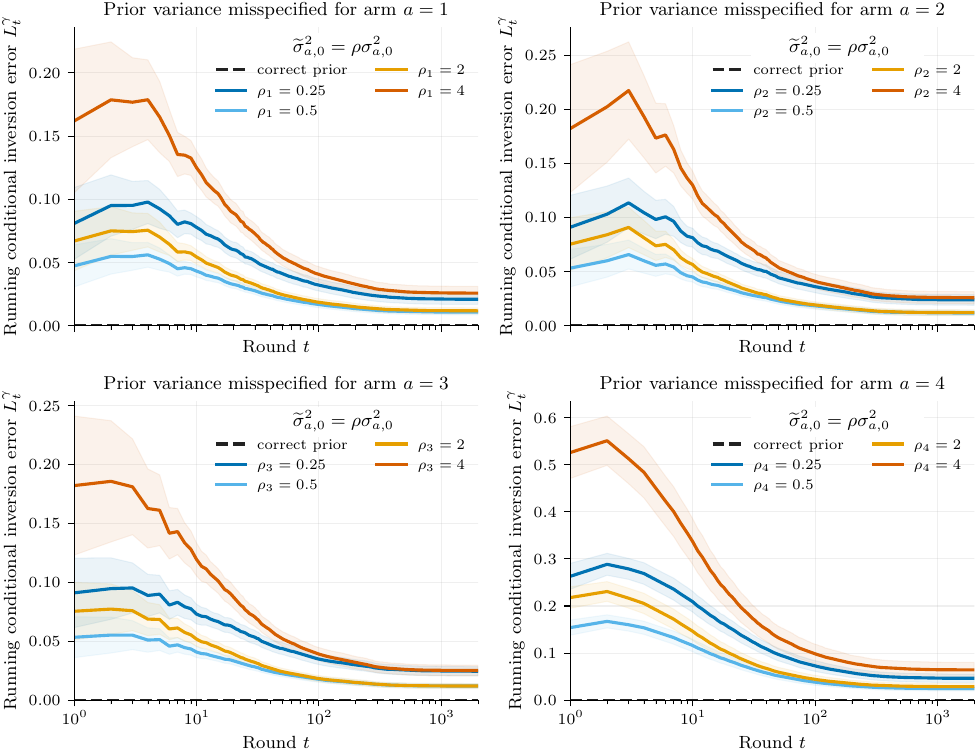}
    \caption{Running conditional inversion error under prior-variance
    misspecification for $\Phi$-LS. Each panel perturbs the inverse learner's
    prior variance for the indicated arm; $\rho$ denotes the multiplier in
    $\widetilde{\sigma}_{a,0}^{2}=\rho\sigma_{a,0}^{2}$. The dashed curve is
    the correctly specified prior baseline.}
    \label{fig:prior-variance-phi}
\end{figure*}

\begin{figure*}[t]
    \centering
    \includegraphics[width=\textwidth]
    {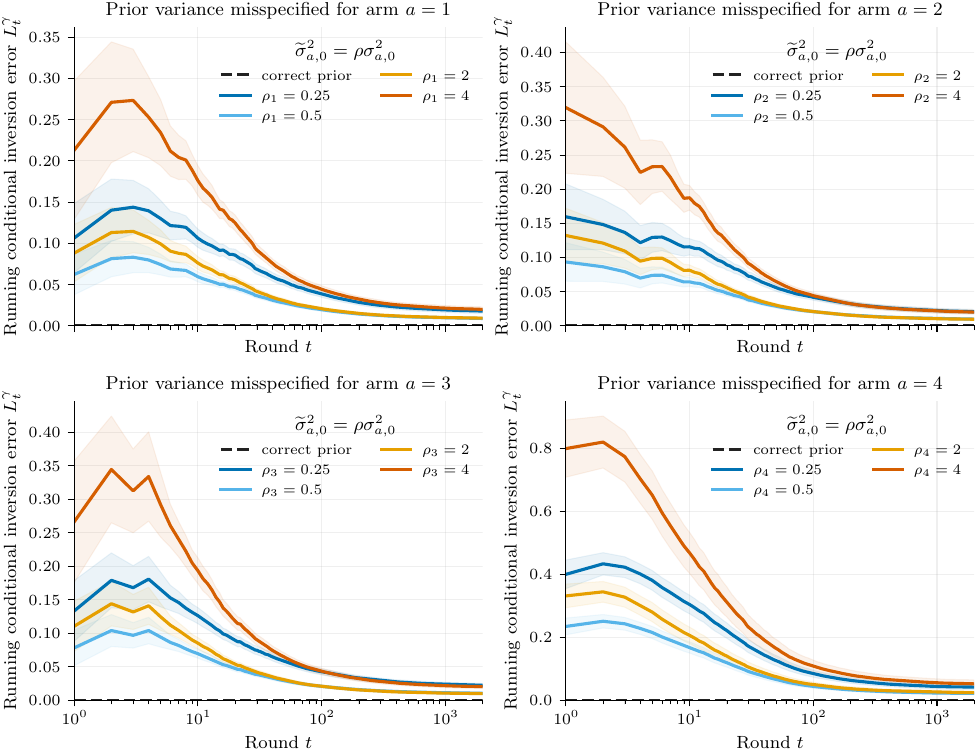}
    \caption{Running conditional inversion error under prior-variance
    misspecification for $\Psi$-LS. The setup and conventions are the same as
    in Figure~\ref{fig:prior-variance-phi}.}
    \label{fig:prior-variance-psi}
\end{figure*}


\end{document}